\documentclass[11pt]{article}
\usepackage{enumerate}
\usepackage{pdfsync}
\usepackage[OT1]{fontenc}

\usepackage{url}            
\usepackage{booktabs}       
\usepackage{amsfonts}       
\usepackage{nicefrac}       
\usepackage{microtype}      
\usepackage{smile}
\usepackage{color}

\usepackage[colorlinks,
linkcolor=red,
anchorcolor=blue,
citecolor=blue
]{hyperref}

\usepackage{fullpage}
\usepackage{setspace}
\usepackage{tabularx}
\usepackage{float}
\usepackage{wrapfig,lipsum}
\usepackage{enumitem}
\usepackage{natbib}
\usepackage{multirow} 
\usepackage{xcolor}

\makeatletter
\newcommand*{\rom}[1]{\expandafter\@slowromancap\romannumeral #1@}
\makeatother

\newcommand{\la}{\langle}
\newcommand{\ra}{\rangle}
\def \layerLen {L}
\def \epochNum {q}
\def \nnweight{\wb}

\def \BB{\mathbb{B}}
\def \bufferLen{H}
\def \nnIter{n}
\def \empR{\hat r}
\def \fosPara{\ell_{\text{Lip}}}
\def \algname{\text{Neural-LinUCB}}

\allowdisplaybreaks

\title{\huge Neural Contextual Bandits with Deep Representation and Shallow Exploration}

\author{%
  Pan Xu\thanks{Department of Computer Science, University of California, Los Angeles, Los Angeles, CA 90095; e-mail: {\tt panxu@cs.ucla.edu}}
  ~~~and~~~
  Zheng Wen\thanks{DeepMind, Mountain View, CA 94043; e-mail: {\tt zhengwen@google.com}}
  ~~~and~~~
  Handong Zhao\thanks{Adobe Research, San Jose, CA 95110; e-mail: {\tt hazhao@adobe.com}}
  ~~~and~~~
  Quanquan Gu\thanks{Department of Computer Science, University of California, Los Angeles, Los Angeles, CA 90095; e-mail: {\tt qgu@cs.ucla.edu}}   
}

\begin{document}

\date{}
\maketitle

\begin{abstract}
  We study a general class of contextual bandits, where each context-action pair is associated with a raw feature vector, but the reward generating function is unknown. We propose a novel learning algorithm that transforms the raw feature vector using the last hidden layer of a deep ReLU neural network (deep representation learning), and uses an upper confidence bound (UCB) approach to explore in the last linear layer (shallow exploration). We prove that under standard assumptions, our proposed algorithm achieves $\tilde{O}(\sqrt{T})$ finite-time regret, where $T$ is the learning time horizon. Compared with existing neural contextual bandit algorithms, our approach is computationally much more efficient since it only needs to explore in the last layer of the deep neural network.
\end{abstract}

\section{Introduction}\label{sec:intro}
Multi-armed bandits (MAB) \citep{auer2002finite,audibert2009exploration,lattimore2020bandit} are a class of online decision-making problems where an agent needs to learn to maximize its expected cumulative reward while repeatedly interacting with a partially known environment.
%
%
Based on a bandit algorithm (also called a strategy or policy), in each round, the agent adaptively chooses an arm, and then observes and receives a reward associated with that arm.
Since only the reward of the chosen arm will be observed (bandit information feedback), a good bandit algorithm has to deal with the exploration-exploitation dilemma: trade-off between pulling the best arm based on existing knowledge/history data (exploitation) and trying the arms that have not been fully explored (exploration).

In many real-world applications, the agent will also be able to access detailed contexts associated with the arms. For example, when a company wants to choose an advertisement to present to a user, the recommendation will be much more accurate if the company takes into consideration the contents, specifications, and other features of the advertisements in the arm set as well as the profile of the user. To encode the contextual information, contextual bandit models and algorithms have been developed, and widely studied both in theory and in practice \citep{danistochastic2008stochastic,rusmevichientong2010linearly,li2010contextual,chu2011contextual,abbasi2011improved}. 
Most existing contextual bandit algorithms assume that the expected reward of an arm at a context is a linear function in a known context-action feature vector, which leads to many useful algorithms such as LinUCB \citep{chu2011contextual}, OFUL \citep{abbasi2011improved}, etc. The representation power of the linear model can be limited in applications such as marketing, social networking, clinical studies, etc., where the rewards are usually counts or binary variables. The linear contextual bandit problem has also been extended to richer classes of parametric bandits such as the generalized linear bandits \citep{filippi2010parametric,li2017provably} and kernelised bandits \citep{valko2013finite,chowdhury2017kernelized}. 

With the resurgence of deep neural networks and their phenomenal performances in many machine learning tasks \citep{lecun2015deep,goodfellow2016deep}, there has emerged a line of work that employs deep neural networks to increase the representation power of contextual bandit algorithms. \citet{zhou2019neural} developed the NeuralUCB algorithm, which can be viewed as a direct extension of linear contextual bandits \citep{abbasi2011improved}, where they use the output of a deep neural network with the feature vector as input to approximate the reward. \citet{zhang2020neural} adapted neural networks in Thompson Sampling \citep{thompson1933likelihood,chapelle2011empirical} for both exploration and exploitation and proposed the NeuralTS algorithm. For a fixed time horizon $T$, it has been proved that both NeuralUCB and NeuralTS achieve a $O(\tilde d\sqrt{T})$ regret  bound, where $\tilde d$ is the effective dimension of a neural tangent kernel matrix which can potentially scale with $O(TK)$ for $K$-armed bandits. This high complexity is mainly due to their exploration over the entire neural network parameter space. A more realistic and efficient way of using deep neural networks in contextual bandits may be to just explore different arms using the last layer as the exploration parameter. More specifically, \citet{riquelme2018deep} provided an extensive empirical study of benchmark algorithms for contextual-bandits through the
lens of Thompson Sampling \citep{thompson1933likelihood,chapelle2011empirical}. They found that decoupling representation learning and uncertainty estimation
improves performance.

In this paper, we study a new neural contextual bandit algorithm, which learns a mapping to transform the raw features associated with each context-action pair using a deep neural network, and then performs an upper confidence bound (UCB)-type (shallow) exploration over the last layer of the neural network. We prove a sublinear regret of the proposed algorithm by exploiting the UCB exploration techniques in linear contextual bandits \citep{abbasi2011improved} and the analysis of deep overparameterized neural networks using neural tangent kernel \citep{jacot2018neural}. Our theory confirms the effectiveness of decoupling the deep representation learning and the UCB exploration in contextual bandits \citep{riquelme2018deep,zahavy2019deep}. 

\noindent\textbf{Contributions} we summarize the main contributions of this paper as follows.
\begin{itemize}[leftmargin=*]
    \item We propose a contextual bandit algorithm, $\algname$, for solving a general class of contextual bandit problems without any assumption on the structure of the reward generating function. The proposed algorithm learns a deep representation to transform the raw feature vectors and performs UCB-type exploration in the last layer of the neural network, which we refer to as deep representation and shallow exploration. Compared with  \cite{zhou2019neural}, our algorithm is much more computationally efficient in practice.
    \item We prove a $\tilde O(\sqrt{T})$ regret for the proposed $\algname$ algorithm, which matches the sublinear regret of linear contextual bandits \citep{chu2011contextual,abbasi2011improved}. To the best of our knowledge, this is the first work that theoretically shows the Neural-Linear schemes of contextual bandits are able to converge, which validates the empirical observation by \citet{riquelme2018deep}.
    \item We conduct experiments  on contextual bandit problems based on real-world datasets, which demonstrates the good performance and computational efficiency of $\algname$ over NeuralUCB and well aligns with our theory.
\end{itemize}

\subsection{Additional related work}


There is a line of related work to ours on the recent advance in the optimization and generalization analysis of deep neural networks. In particular, \citet{jacot2018neural} first introduced the neural tangent kernel (NTK) to characterize the training dynamics of network outputs in the infinite width limit. From the notion of NTK, a fruitful line of research emerged and showed that loss functions of deep neural networks trained by (stochastic) gradient descent can converge to the global minimum \citep{du2018gradient,allen2019convergence,du2019gradient,zou2018stochastic,zou2019improved}. The generalization bounds for overparameterized deep neural networks are also established in \citet{arora2019fine,arora2019exact,allen2019learning,cao2019generalization,cao2019generalization2}. Recently, the NTK based analysis is also extended to the study of sequential decision problems including NeuralUCB \citep{zhou2019neural}, and reinforcement learning algorithms  \citep{cai2019neural,liu2019neural,Wang2020Neural,xu2020finite}.

Our algorithm is also different from \citet{langford2008epoch,agarwal2014taming} which reduce the bandit problem to supervised learning. 
Moreover, their algorithms need to access an oracle that returns the optimal policy in a policy class given a sequence of context and reward vectors, whose regret depends on the VC-dimension of the policy class. 

\noindent\textbf{Notation} We use $[k]$ to denote a  set $\{1,\ldots,k\}$, $k\in\NN^+$. $\|\xb\|_2=\sqrt{\xb^{\top}\xb}$ is the Euclidean norm of a vector $\xb\in\RR^d$. For a matrix $\Wb\in\RR^{m\times n}$, we denote by $\|\Wb\|_2$ and $\|\Wb\|_F$ its operator norm and Frobenius norm respectively. 
For a semi-definite matrix $\Ab\in\RR^{d\times d}$ and a vector $\xb\in\RR^d$, we denote the Mahalanobis norm as $\|\xb\|_{\Ab}=\sqrt{\xb^{\top}\Ab\xb}$. 
Throughout this paper, we reserve the notations $\{C_i\}_{i=0,1,\ldots}$ to represent absolute positive constants that are independent of problem parameters such as dimension, sample size, iteration number, step size, network length and so on. 
The specific values of $\{C_i\}_{i=0,1,\ldots}$ can be different in different context. For a parameter of interest $T$ and a function $f(T)$, we use notations such as $O(f(T))$ and $\Omega(f(T))$ to hide  constant factors and $\tilde O(f(T))$ to hide constant and logarithmic dependence of $T$. 

\section{Preliminaries}
In this section, we provide the background of contextual bandits and deep neural networks.

\subsection{Linear contextual bandits}\label{sec:lin_ucb}
A contextual bandit is characterized by a tuple $(\cS, \cA, r)$, where $\cS$ is the context (state) space, $\cA$ is the arm (action) space, and $r$ encodes the unknown \emph{reward generating function} at all context-arm pairs. A learning agent, who knows $\cS$ and $\cA$ but does not know the true reward $r$ (values bounded in $(0,1)$ for simplicity), needs to interact with the contextual bandit for $T$ rounds. At each round $t=1, \ldots, T$, the agent first observes a context $s_t \in \cS$ chosen by the environment; then it needs to adaptively select an arm $a_t \in \cA$ based on its past observations; finally it receives a reward $\hat{r}_t(\xb_{s,a_t})=r(\xb_{s,a_t})+\xi_t $, 
where $\xb_{s,a} \in \RR^d$ is a known feature vector for context-arm pair $(s, a) \in \cS \times \cA$, and $\xi_t$ is a random noise with zero mean. The agent's objective is to maximize its expected total reward over these $T$ rounds, which is equivalent to minimizing the pseudo regret \citep{audibert2009exploration}:
\begin{align}\label{def:pseudo_regret}
R_T=\EE\bigg[\sum_{t=1}^{T}\big( \hat r(\xb_{s_t, a_t^*}) - \hat r(\xb_{s_t, a_t}) \big)\bigg], 
\end{align}
where $a_t^* \in \argmax_{a \in \cA} \{r(\xb_{s_t, a})=\EE[\hat r(\xb_{s_t,a})]\}$. To simplify the exposition, we use $\xb_{t,a}$ to denote $\xb_{s_t, a}$ since the context only depends on the round index $t$ in most bandit problems, and we assume $\cA = [K]$.
%
%

In some practical problems, the agent has a prior knowledge that the reward-generating function $r$ has some specific parametric form. For instance, in linear contextual bandits, the agent knows that $r(\xb_{s,a}) = \xb_{s,a}^\top \btheta^*$ for some unknown weight vector $\btheta^*\in\RR^d$. One provably sample efficient algorithm for linear contextual bandits is Linear Upper Confidence Bound (LinUCB) \citep{abbasi2011improved}. Specifically, at each round $t$, LinUCB chooses action by the following strategy
\[
a_t=\argmax_{a \in [K]} \left\{ \xb^{\top}_{t, a}\btheta_{t} +\alpha_t \|\xb_{t,a}\|_{\Ab_{t}^{-1}} \right\},
\]
where $\btheta_{t}$ is a point estimate of $\btheta^*$, $\Ab_{t}=\lambda \mathbf{I}+\sum_{i=1}^{t}\xb_{i,a_i}\xb_{i,a_i}^{\top}$ with some $\lambda>0$ is a matrix defined based on the historical context-arm pairs, and $\alpha_t>0$ is a tuning parameter that controls the exploration rate in LinUCB. 

\subsection{Deep neural networks}\label{sec:neural_bandit}
In this paper, we use $f(\xb)$ to denote a neural network with input data $\xb\in\RR^d$.
Let $\layerLen$ be the number of hidden layers and $\Wb_{l}\in\RR^{m_{l}\times m_{l-1}}$ be the weight matrices in the $l$-th layer, where $l=1,\ldots,L$, $m_1=\ldots=m_{\layerLen-1}=m$ and $m_0=m_{\layerLen}=d$. 
Then a $\layerLen$-hidden layer neural network is defined as
\begin{align}\label{eq:def_network}
    f(\xb)=\sqrt{m}\btheta^{*\top}\sigma_{\layerLen}(\Wb_{\layerLen}\sigma_{\layerLen-1}(\Wb_{\layerLen-1}\cdots\sigma_1(\Wb_1\xb)\cdots)),
\end{align}
where $\sigma_l$ is an activation function and $\btheta^*\in\RR^{d}$ is the weight of the output layer. To simplify the presentation, we will assume $\sigma_1=\sigma_2=\ldots=\sigma_{\layerLen}=\sigma$ is the ReLU activation function, i.e.,  $\sigma(x)=\max\{0,x\}$ for $x\in\RR$.  We denote $\nnweight=(\text{vec}(\Wb_1)^{\top},\ldots,\text{vec}(\Wb_{\layerLen})^{\top})^{\top}$, which is the concatenation of the vectorized weight parameters of all hidden layers of the neural network. We also write $f(\xb;\btheta^*,\nnweight)=f(\xb)$ in order to explicitly specify the weight parameters of neural network $f$. It is easy to show that the dimension $p$ of vector $\nnweight$ satisfies $p=(\layerLen-2)m^2+2md$. To simplify the notation, we define $\bphi(\xb;\nnweight)$ as the output of the $L$-th hidden layer of neural network $f$.  
\begin{align}\label{eq:def_hidden_layer}
    \bphi(\xb;\nnweight)=\sqrt{m}\sigma(\Wb_{\layerLen}\sigma(\Wb_{\layerLen-1}\cdots\sigma(\Wb_1\xb)\cdots)).
\end{align}
Note that $\bphi(\xb;\nnweight)$ itself can also be viewed as a neural network with vector-valued outputs.

\section{Deep Representation and  Shallow  Exploration}\label{sec:algorithm}

The linear parametric form in linear contextual bandit might produce biased estimates of the reward due to the lack of representation power \citep{snoek2015scalable,riquelme2018deep}. 
In contrast, it is well known that deep neural networks are powerful enough to approximate an arbitrary function \citep{cybenko1989approximation}. Therefore, it would be a natural extension for us to guess that the reward generating function $r(\cdot)$ can be represented by a deep neural network. Nonetheless, deep neural networks usually have a prohibitively large dimension for weight parameters, which makes the exploration  in neural networks based UCB algorithm  inefficient \citep{kveton2020randomized,zhou2019neural}. 

In this work, we study a more realistic setting, where the hidden layers of a deep neural network are used to represent the features and the exploration is only performed in the last layer of the neural network \citep{riquelme2018deep,zahavy2019deep}. In particular, we assume that the reward generating function $r(\cdot)$ can be expressed as the inner product between a deep represented feature vector and an exploration weight parameter, namely, $r(\cdot)=\la\btheta^*,\bpsi(\cdot)\ra$, where $\btheta^*\in\RR^d$ is some  weight parameter and $\bpsi(\cdot)$ is an unknown feature mapping. This decoupling of the representation and the exploration will achieve the best of both worlds: efficient exploration in shallow (linear) models and high expressive power of deep models. To learn the unknown feature mapping, we propose to use a neural network to approximate it. In what follows, we will describe a neural contextual bandit algorithm that uses the output of the last hidden layer of a neural network to transform the raw feature vectors (\emph{deep representation}) and performs UCB-type exploration in the last layer of the neural network (\emph{shallow exploration}). Since the exploration is performed only in the last linear layer, we call this procedure $\algname$, which is displayed in Algorithm \ref{alg:deepUCB}.

Specifically, in round $t$, the agent receives an action set with raw features $\cX_t=\{\xb_{t,1},\ldots,\xb_{t,K}\}$. Then the agent chooses an arm $a_t$ that maximizes the following upper confidence bound:
\begin{align}\label{eq:neural_ucb_time_t}
\begin{split}
    a_{t}&=\argmax_{k\in[K]} \Big\{\la \bphi(\xb_{t,k};\nnweight_{t-1}),\btheta_{t-1}\ra+\alpha_t\|\bphi(\xb_{t,k};\nnweight_{t-1})\|_{\Ab_{t-1}^{-1}}\Big\},
\end{split}    
\end{align}
where $\btheta_{t-1}$ is a point estimate of the unknown weight in the last layer, $\bphi(\xb;\nnweight)$ is defined as in~\eqref{eq:def_hidden_layer}, $\nnweight_{t-1}$ is an estimate of all the weight parameters in the hidden layers of the neural network,   $\alpha_t>0$ is the algorithmic parameter controlling the exploration, and $\Ab_t$ is a matrix defined based on historical transformed features: \begin{align}\label{eq:def_design_matrix}
    \Ab_{t} =\lambda\Ib+\sum_{i=1}^t\bphi(\xb_{i,a_i};\nnweight_{i-1})\bphi(\xb_{i,a_i};\nnweight_{i-1})^{\top},
\end{align} 
and $\lambda>0$. After pulling arm $a_t$, the agent will observe a noisy reward $\empR_t:=\empR(\xb_{t,a_t})$ defined as 
\begin{align}\label{eq:nn_model}
   \empR(\xb_{t,k})=r(\xb_{t,k})+\xi_t ,
\end{align}
where $\xi_t$ is an independent  $\nu$-subGaussian random noise for some $\nu>0$ and $r(\cdot)$ is an unknown reward function. In this paper, we will interchangeably use notation $\empR_t$ to denote the reward received at the $t$-th step and an equivalent notation $\empR(\xb)$ to express its dependence on the feature vector $\xb$.

Upon receiving the  reward $\empR_t$, the agent updates its estimate $\btheta_t$ of the output  layer weight by using the same $\ell^2$-regularized least-squares estimate in linear contextual bandits \citep{abbasi2011improved}. In particular, we have
\begin{align}
    \btheta_{t}=\Ab_{t}^{-1}\bbb_{t},
\end{align}
where $\bbb_{t} =\sum_{i=1}^t\empR_i\phi(\xb_{i,a_i};\nnweight_{i-1})$. 

To save the computation, the neural network $\bphi(\cdot;\nnweight_t)$ will be  updated once every $\bufferLen$ steps. 
Therefore, we have $\nnweight_{(q-1)\bufferLen+1}=\ldots=\nnweight_{q\bufferLen}$ for $q=1,2,\ldots$. We call the time steps $\{(q-1)\bufferLen+1, \ldots, q\bufferLen\}$ an epoch with length $\bufferLen$. At time step $t=\bufferLen \epochNum$, for any $q=1,2,\ldots$, Algorithm \ref{alg:deepUCB} will retrain the neural network based on all the historical data. In Algorithm \ref{alg:update_nn}, our goal is to minimize the following empirical loss function:
\begin{align}\label{eq:def_loss_nn}
    \cL_{\epochNum}(\nnweight)=\sum_{i=1}^{\epochNum\bufferLen}\big(\btheta_{i}^{\top}\phi(\xb_{i,a_i};\nnweight)-\empR_i\big)^2.
\end{align}
In practice, one can further save computational cost by only feeding data $\{\xb_{i,a_i},\empR_i,\btheta_i\}_{i=(\epochNum-1)\bufferLen+1}^{q\bufferLen}$ from the $\epochNum$-th epoch into Algorithm \ref{alg:update_nn} to  update the parameter $\nnweight_t$, which does not hurt the performance since the historical information has been encoded into the estimate of $\btheta_i$. In this paper, we will perform the following gradient descent step
\begin{align*}
    \nnweight_{\epochNum}^{(s)}=\nnweight_{\epochNum}^{(s-1)}-\eta_{\epochNum}\nabla_{\nnweight}\cL_{\epochNum}(\nnweight^{(s-1)}).
\end{align*}
for $s=1,\ldots,\nnIter$, where $\nnweight_{\epochNum}^{(0)}=\nnweight^{(0)}$ is chosen as the same random initialization point. We will discuss more about the initial point $\nnweight^{(0)}$ in the next paragraph.
Then Algorithm \ref{alg:update_nn} outputs $\nnweight_{\epochNum}^{(\nnIter)}$ and we set it as the updated weight parameter $\nnweight_{\bufferLen \epochNum+1}$ in Algorithm \ref{alg:deepUCB}. 
In the next round, the agent will receive another action set $\cX_{t+1}$ with raw feature vectors and repeat the above steps to choose the sub-optimal arm and update estimation for contextual parameters. 

\noindent\textbf{Initialization:} Recall that  $\nnweight$ is the collection of all hidden layer weight parameters of the neural network. We will follow the same initialization scheme as used in \citet{zhou2019neural}, where each entry of the weight matrices follows some Gaussian distribution. Specifically, for any $l\in\{1,\ldots,\layerLen-1\}$, we set $\Wb_{l}=
\begin{bmatrix}
\Wb&\zero\\
\zero&\Wb
\end{bmatrix}$, 
where each entry of $\Wb$ follows distribution $N(0,4/m)$ independently; for $\Wb_{\layerLen}$, we set it as $\begin{bmatrix}
\Vb&-\Vb\\
\end{bmatrix}$,
where each entry of $\Vb$ follows distribution $N(0,2/m)$ independently. 

\noindent\textbf{Comparison with LinUCB and NeuralUCB:}
Compared with linear contextual bandits in Section \ref{sec:lin_ucb}, Algorithm \ref{alg:deepUCB} has a distinct feature that it learns a deep neural network to obtain a deep representation of the raw data vectors and then performs UCB exploration. This deep representation allows our algorithm to characterize more intrinsic and latent information about the raw data $\{\xb_{t,k}\}_{t\in[T],k\in[K]}\subset\RR^d$. However, the increased complexity of the feature mapping $\bphi(\cdot;\nnweight)$ also introduces great hardness in training. For instance, a  recent work by \cite{zhou2019neural} also studied the neural contextual bandit problem, but different from \eqref{eq:neural_ucb_time_t}, their algorithm (NeuralUCB) performs the UCB exploration on the entire network parameter space, which is $\RR^{p+d}$. Note that in \cite{zhou2019neural}, they need to compute the inverse of a matrix $\Zb_t\in\RR^{(p+d)\times(p+d)}$, which is defined in a similar way to the matrix $\Ab_t$ in our paper except that $\Zb_t$ is defined based on the gradient of the network instead of the output of the last hidden layer as in \eqref{eq:def_design_matrix}. In sharp contrast, $\Ab_t$ in our paper is only of size $d\times d$ and thus is much more efficient and practical in implementation, which will be seen from our experiments in later sections. 

We note that there is also a  similar algorithm to our $\algname$ presented in \citet{deshmukh2020self}, where they studied the self-supervised learning loss in contextual bandits with neural network representation for  computer vision problems. However, no regret analysis has been provided. When the feature mapping $\bphi(\cdot;\nnweight)$ is an identity function, the problem reduces to linear contextual bandits where we directly use $\xb_t$ as the feature vector. In this case, it is easy to see that Algorithm~\ref{alg:deepUCB} reduces to LinUCB \citep{chu2011contextual} since we do not need to learn the representation parameter $\nnweight$ anymore.

\noindent\textbf{Comparison with Neural-Linear:} Our algorithm is also similar to the Neural-Linear algorithm studied in \cite{riquelme2018deep}, which trains a deep neural network to learn a representation of the raw feature vectors, and then uses a Bayesian linear regression to estimate the uncertainty in the bandit problem. The difference between Neural-Linear \citep{riquelme2018deep} and  $\algname$ in this paper lies in the specific exploration strategies: Neural-Linear uses posterior sampling to estimate the weight parameter $\btheta^*$ via Bayesian linear regression, whereas  $\algname$  adopts upper confidence bound based techniques to estimate the weight $\btheta^*$. Nevertheless, both algorithms share the same idea of deep representation and shallow exploration, and we view our $\algname$ algorithm as one instantiation of the Neural-Linear scheme proposed in \cite{riquelme2018deep}.

\begin{algorithm}[t]
\caption{Deep Representation and Shallow Exploration ($\algname$)} 
\label{alg:deepUCB}
\begin{algorithmic}[1]
\STATE\textbf{Input}: regularization parameter $\lambda>0$, number of total steps $T$, episode length $H$, exploration parameters $\{\alpha_t>0\}_{t\in[T]}$
\STATE\textbf{Initialization:} $\Ab_{0}=\lambda\Ib$, $\bbb_0=\zero$; entries of $\btheta_{0}$ follow $N(0,1/d)$, and $\nnweight^{(0)}$ is initialized as described in Section \ref{sec:algorithm};  $\epochNum=1$; $\nnweight_0=\nnweight^{(0)}$
\FOR{$t=1,\ldots,T$}
\STATE receive feature vectors $\{\xb_{t,1},\ldots,\xb_{t,K}\}$
\STATE choose arm 
$a_t=\argmax_{k\in[K]} \btheta_{t-1}^{\top}\bphi(\xb_{t,k};\nnweight_{t-1})$ $+\alpha_t\|\bphi(\xb_{t,k};\nnweight_{t-1})\|_{\Ab_{t-1}^{-1}}$\label{algline:ucb_rule}, and obtain reward $\empR_t$
\STATE update $\Ab_{t}$ and $\bbb_{t}$ as follows: \\
    \hspace{0.1in}$\Ab_{t}=
    \Ab_{t-1}+\bphi(\xb_{t,a_t};\nnweight_{t-1})\bphi(\xb_{t,a_t};\nnweight_{t-1})^{\top},$
    \\ 
    \hspace{0.1in}$\bbb_{t}=
    \bbb_{t-1}+\empR_t\bphi(\xb_{t,a_t};\nnweight_{t-1}),$
\STATE update $\btheta_{t}=\Ab_{t}^{-1}\bbb_{t}$
\IF{mod$(t,\bufferLen)=0$}
\STATE $\nnweight_{t}$ $\leftarrow$ output of Algorithm \ref{alg:update_nn} 
\STATE $\epochNum=\epochNum+1$
\ELSE
\STATE $\nnweight_{t}=\nnweight_{t-1}$
\ENDIF
\ENDFOR
\STATE\textbf{Output}
\end{algorithmic} 
\end{algorithm}

\begin{algorithm}
\caption{Update Weight Parameters with Gradient Descent}
\label{alg:update_nn}
\begin{algorithmic}[1]
\STATE \textbf{Input:}  initial point $\nnweight_{\epochNum}^{(0)}=\nnweight^{(0)}$, maximum iteration number $\nnIter$, step size $\eta_{\epochNum}$,  and loss function defined in \eqref{eq:def_loss_nn}.
\FOR{$s=1,\ldots,\nnIter$}
\STATE $\nnweight_{\epochNum}^{(s)}=\nnweight_{\epochNum}^{(s-1)}-\eta_{\epochNum}\nabla_{\nnweight}\cL_{\epochNum}(\nnweight_{\epochNum}^{(s-1)}).
$
\ENDFOR
\RETURN $\nnweight_{\epochNum}^{(\nnIter)}$
\end{algorithmic}
\end{algorithm}

\section{Main Theory}

To analyze the regret bound of Algorithm \ref{alg:deepUCB}, we first lay down some important assumptions on the neural contextual bandit model.
\begin{assumption}\label{asp:nondegen}
For all $i\geq 1$ and $k\in[K]$, we assume that $\|\xb_{i,k}\|_2=1$ and its entries satisfy $[\xb_{i,k}]_j=[\xb_{j,k}]_{j+d/2}$.
\end{assumption}
The assumption that $\|\xb_{i,k}\|_2=1$ is not essential and is only imposed for simplicity, which is also used in \cite{zou2019improved,zhou2019neural}. Finally, the condition on the entries of $\xb_{i,k}$ is also mild since otherwise we could always construct $\xb_{i,k}'=[\xb_{i,k}^{\top},\xb_{i,k}^{\top}]^{\top}/\sqrt{2}$ to replace it. An implication of Assumption \ref{asp:nondegen} is that the initialization scheme in Algorithm \ref{alg:deepUCB} results in $\bphi(\xb_{i,k};\nnweight^{(0)})=\zero$ for all $i\in[T]$ and $k\in[K]$.


We further impose the following stability condition on the spectral norm of the neural network gradient:
\begin{assumption}\label{asp:gradient_fos}
There is a constant $\fosPara>0$ such that it holds 
\begin{align*}
    \bigg\|\frac{\partial\bphi}{\partial\nnweight}(\xb;\nnweight_0)-\frac{\partial\bphi}{\partial\nnweight}(\xb';\nnweight_0)\bigg\|_2\leq\fosPara\|\xb-\xb'\|_2,
\end{align*}
for all $\xb,\xb'\in\{\xb_{i,k}\}_{i\in[T],k\in[K]}$.
\end{assumption}
The inequality in Assumption \ref{asp:gradient_fos} looks like some Lipschitz condition. However, it is worth noting that here the gradient is taken with respect to the neural network weights while the Lipschitz condition is imposed on the feature parameter $\xb$. Similar conditions are widely made in nonconvex optimization \citep{wang2014optimal,balakrishnan2017statistical,xu2017speeding}, in the name of first-order stability, which is essential to derive the convergence of alternating optimization algorithms. Furthermore, Assumption \ref{asp:gradient_fos} is only required on the $TK$ training data points and a specific weight parameter $\nnweight_0$. Therefore, the condition will hold if the raw feature data lie in a benign subspace. A more thorough study of this stability condition is out of the scope of this paper, though it would be an interesting open direction in the theory of deep neural networks.

In order to analyze the regret bound of Algorithm \ref{alg:deepUCB}, we need to characterize the properties of the deep neural network in \eqref{eq:def_network} that is used to represent the feature vectors. Following a recent line of research \citep{jacot2018neural,cao2019generalization,arora2019exact,zhou2019neural}, we define the covariance between two data point $\xb,\yb\in\RR^d$ as follows.
\begin{align}
\begin{split}
    \tilde\bSigma^{(0)}(\xb,\yb)&=\bSigma^{(0)}(\xb,\yb)=\xb^{\top}\yb,\\
    \bLambda^{(l)}(\xb,\yb)&=\begin{bmatrix}
    \bSigma^{l-1}(\xb,\xb)& \bSigma^{l-1}(\xb,\yb)\\
    \bSigma^{l-1}(\yb,\xb)& \bSigma^{l-1}(\yb,\yb)
    \end{bmatrix},\\
    \bSigma^{(l)}(\xb,\yb)&=2\EE_{(u,v)\sim N(\zero,\bLambda^{(l-1)}(\xb,\yb))}[\sigma(u)\sigma(v)],\\
    \tilde\bSigma^{(l)}(\xb,\yb)&=2\tilde\bSigma^{(l-1)}(\xb,\yb)\EE_{u,v}[\dot\sigma(u)\dot\sigma(v)]+\bSigma^{(l)}(\xb,\yb),
\end{split}    
\end{align}
where $(u,v)\sim N(\zero,\bLambda^{(l-1)}(\xb,\yb))$, and $\dot\sigma(\cdot)$ is the derivative of activation function $\sigma(\cdot)$. We denote the neural tangent kernel (NTK) matrix $\Hb$ by a $\RR^{T K\times T K}$ matrix defined on the dataset of all feature vectors $\{\xb_{t,k}\}_{t\in[T],k\in[K]}$. Renumbering $\{\xb_{t,k}\}_{t\in[T],k\in[K]}$ as $\{\xb_{i}\}_{i=1,\ldots,T K}$, then each entry $\Hb_{ij}$ is defined as
\begin{align}\label{eq:def_ntk}
    \Hb_{ij} = \frac{1}{2}\big(\tilde\bSigma^{(\layerLen)}(\xb_i,\xb_j)+\bSigma^{(\layerLen)}(\xb_i,\xb_j)\big),
\end{align}
for all $i,j\in[T K]$. Based on the above definition, we impose the following assumption on $\Hb$.
\begin{assumption}\label{asp:ntk_pd}
The neural tangent kernel defined in \eqref{eq:def_ntk} is positive definite, i.e., $\lambda_{\min}(\Hb)\geq \lambda_0$ for some constant $\lambda_0>0$.
\end{assumption}
Assumption \ref{asp:ntk_pd} essentially requires the neural tangent kernel matrix $\Hb$ to be non-singular, which is a mild condition and also imposed in other related work \citep{du2019gradient,arora2019exact,cao2019generalization,zhou2019neural}. Moreover, it is shown that Assumption \ref{asp:ntk_pd} can be easily derived from Assumption \ref{asp:nondegen} for two-layer ReLU networks \citep{oymak2020towards,zou2019improved}. Therefore, Assumption \ref{asp:ntk_pd} is mild or even negligible given the non-degeneration assumption on the feature vectors. Also note that matrix $\Hb$ is only defined based on layers $l=1,\ldots,\layerLen$ of the neural network, and does not depend on the output layer $\btheta$. It is easy to extend the definition of $\Hb$ to the NTK matrix defined on all layers including the output layer $\btheta$, which would also be positive definite by Assumption \ref{asp:ntk_pd} and the recursion in \eqref{eq:def_ntk}.

Before we present the regret analysis of the neural contextual bandit, we need to  modify the regret defined in \eqref{def:pseudo_regret} to account for the randomness of the neural network initialization. For a fixed time horizon $T$, we define the regret of Algorithm \ref{alg:deepUCB} as follows.
\begin{align}\label{eq:def_regret}
    R_T=\EE\bigg[\sum_{t=1}^{T}\big( \hat r(\xb_{t, a_t^*}) - \hat r(\xb_{t, a_t}) \big)|\nnweight^{(0)}\bigg],
\end{align}
where the expectation is taken over the randomness of the reward noise. Note that $R_T$ defined in~\eqref{eq:def_regret} is still a random variable since the initialization of Algorithm \ref{alg:update_nn} is randomly generated.

Now we are going to present the regret bound of the proposed algorithm.
\begin{theorem}\label{thm:regret}
Suppose Assumptions \ref{asp:nondegen}, \ref{asp:gradient_fos} and \ref{asp:ntk_pd} hold. Assume that $\|\btheta^*\|_2\leq M$ for some positive constant $M>0$. 
For any $\delta\in(0,1)$, let us choose $\alpha_t$ in $\algname$ as
\begin{align*}
    \alpha_t=\nu\sqrt{2\big(d\log(1+t\log (\bufferLen K)/\lambda )+\log (1/\delta)\big)}+\lambda^{1/2}M.
\end{align*}
We choose the step size $\eta_{\epochNum}$ of Algorithm \ref{alg:update_nn} as
\begin{align*}
    \eta_{\epochNum}\leq C_0\big(d^2 m\nnIter T^{5.5}\layerLen^{6}\log(TK/\delta)\big)^{-1},
\end{align*}
and the width of the neural network satisfies $m=\text{poly}(\layerLen,d,1/\delta,\bufferLen,\log(TK/\delta))$.
With probability at least $1-\delta$ over the randomness of the initialization of the neural network, it holds that
\begin{align*}
    R_T&\leq C_1\alpha_T\sqrt{Td\log\Big(1+ \frac{TG^2}{\lambda d}\Big)}+\frac{C_2\fosPara \layerLen^3 d^{5/2}T\sqrt{\log m\log(\frac{1}{\delta})\log(\frac{TK}{\delta})}\|\rb-\tilde\rb\|_{\Hb^{-1}}}{m^{1/6}},
\end{align*}
where $\{C_i\}_{i=0,1,2}$ are absolute constants  independent of the problem parameters, $\|\rb\|_{\Ab}=\sqrt{\rb^{\top}\Ab\rb}$,
 $\rb=(r(\xb_{1}),r(\xb_{2}),\ldots,r(\xb_{T K}))^{\top}\in\RR^{T K}$ and  $\tilde\rb=(f(\xb_{1};\btheta_{0},\nnweight_0),\ldots,f(\xb_{T K};\btheta_{T-1},\nnweight_{T-1}))^{\top}\in\RR^{T K}$.
\end{theorem}
\begin{remark}
Theorem \ref{thm:regret} shows that the regret of Algorithm \ref{alg:deepUCB} can be bounded by two parts: the first part is of order $\tilde O(\sqrt{T})$, which resembles the regret bound of linear contextual bandits \citep{abbasi2011improved}; the second part is of order $\tilde O(m^{-1/6}T\sqrt{(\rb-\tilde\rb)^{\top}\Hb^{-1}(\rb-\tilde\rb)})$, which depends on the estimation error of the neural network $f$ for the reward generating function $r$ and the neural tangent kernel $\Hb$.
\end{remark}
\begin{remark}
For the ease of presentation, let us denote $\cE:=\|\rb-\tilde\rb\|_{\Hb^{-1}}$. If we have  $\cE=O(1)$, the total regret in Theorem \ref{thm:regret} becomes $\tilde O(m^{-1/6}T)$. If we further choose a sufficiently overparameterized neural network with $m\geq T^3$, then the regret reduces to $\tilde O(\sqrt{T})$ which matches the regret of linear contextual bandits \citep{abbasi2011improved}. 
We remark that there is a similar  assumption in \cite{zhou2019neural} where they assume that $\rb^{\top}\Hb^{-1}\rb$ can be upper bounded by a constant. They show that this term can be bounded by the RKHS norm of $\rb$ if it belongs to the RKHS induced by the neural
tangent kernel \citep{arora2019fine,arora2019exact,lee2019wide}.
In addition, $\cE$ here is the difference between the true reward function and the neural network function, which can also be small if the deep neural network function well approximates the reward generating function $r(\cdot)$. 
\end{remark}

\section{Experiments}
In this section, we provide empirical evaluations of the proposed $\algname$ algorithm. As we have discussed in Section \ref{sec:algorithm}, $\algname$ could be viewed as an instantiation of the Neural-Linear scheme studied in \cite{riquelme2018deep} except that we use the UCB exploration instead of the posterior sampling exploration therein. Note that there has been an extensive comparison \citep{riquelme2018deep} of the Neural-Linear methods with many other baselines such as greedy algorithms, Variational Inference, Expectation-Propagation, Bayesian Non-parametrics and so on. Therefore, we do not seek a thorough empirical comparison of $\algname$ with all existing bandits algorithms. We refer readers who are interested in the performance of Neural-Linear methods with deep representation and shallow exploration compared with a vast of baselines in the literature to the benchmark study by \citet{riquelme2018deep}. In this experiment, we only aim to show the advantages of our algorithm over the following baselines: (1) Neural-Linear \citep{riquelme2018deep}; (2) LinUCB \citep{chu2011contextual}, which does not have a \textit{deep representation} of the feature vectors; and (3) NeuralUCB \citep{zhou2019neural}, which performs UCB exploration on all the parameters of the neural network instead of the \textit{shallow exploration} used in our paper.


\noindent\textbf{Datasets:} we evaluate the performances of all algorithms on bandit problems created from real-world data. Specifically, we use datasets \emph{(Shuttle) Statlog}, \emph{Magic} and \emph{Covertype}   from UCI machine learning repository \citep{dua2019uci}, of which the details are presented in Table \ref{table:dataset}. In Table \ref{table:dataset}, each instance represents a feature vector $\xb\in\RR^{d}$ that is associated with one of the $K$ arms, and dimension $d$ is the number of attributes in each instance.

\begin{table}[ht]
    \centering
    \caption{Specifications of  datasets from the UCI machine learning repository used in this paper.     \label{table:dataset}}
    \begin{tabular}{lccc}
    \toprule
    &\emph{ Statlog} & \emph{Magic} &\emph{Covertype}\\
    \midrule
    Number of attributes     & 9& 11 & 54\\
    Number of arms & 7& 2& 7\\
    Number of instances     & 58,000& 19,020& 581,012 \\    
    \bottomrule
    \end{tabular}
\end{table}

\begin{figure*}[h]
\subfigure[Statlog]{\label{fig:statlog}\includegraphics[height=0.24\textwidth]{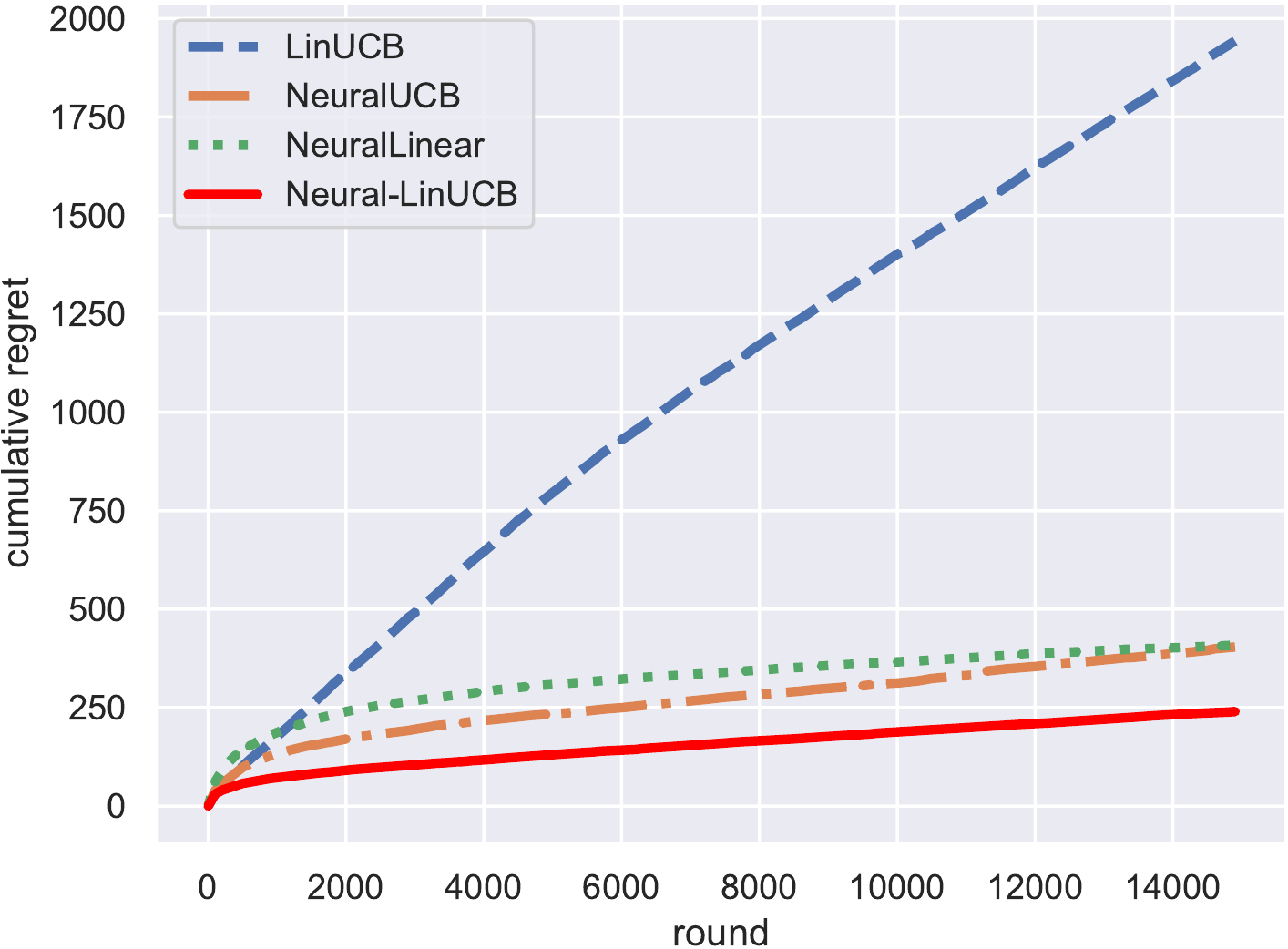}}
\subfigure[Magic]{\label{fig:magic}\includegraphics[height=0.24\textwidth]{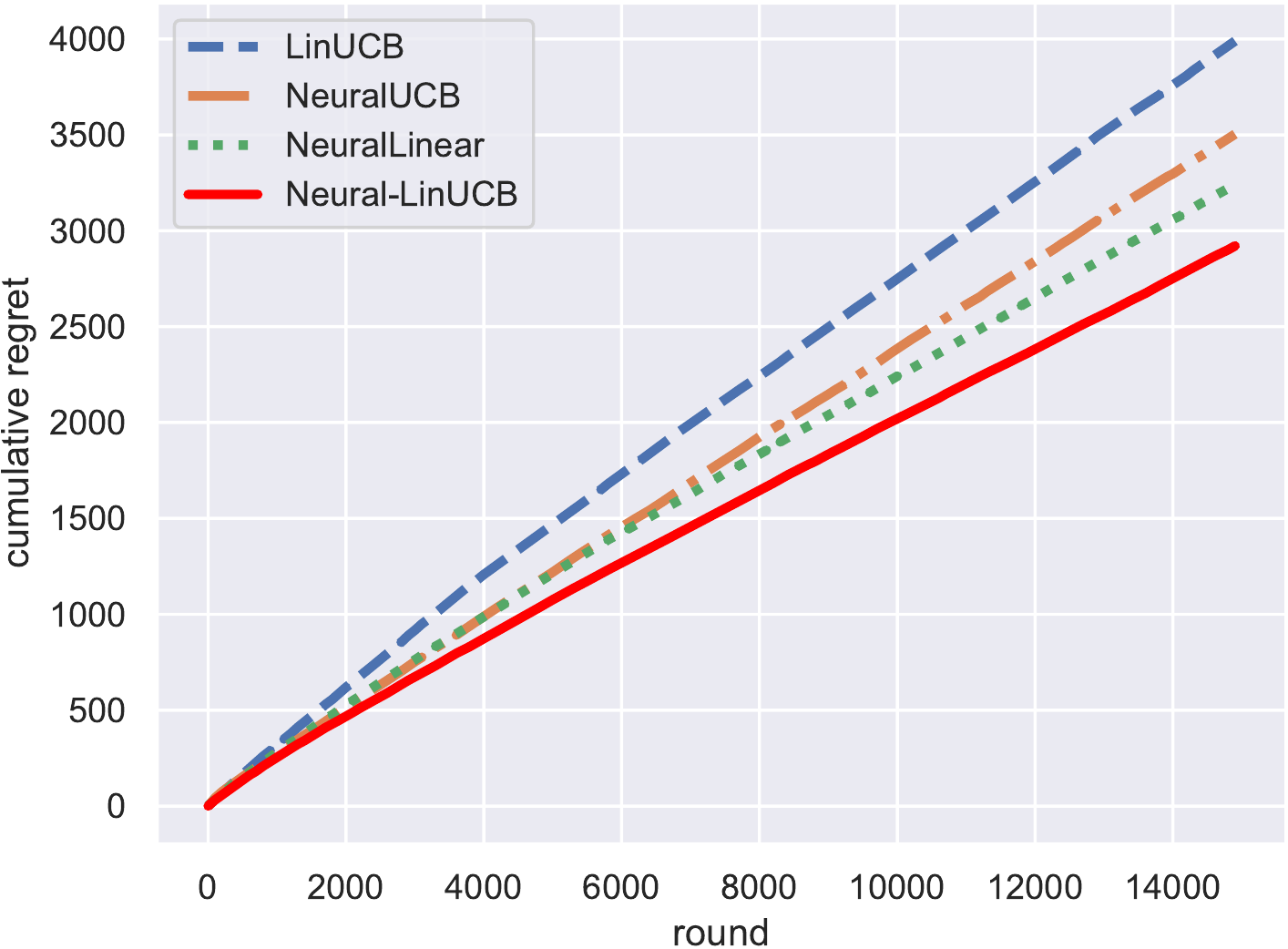}}
\subfigure[Covertype]{\label{fig:covertype}\includegraphics[height=0.24\textwidth]{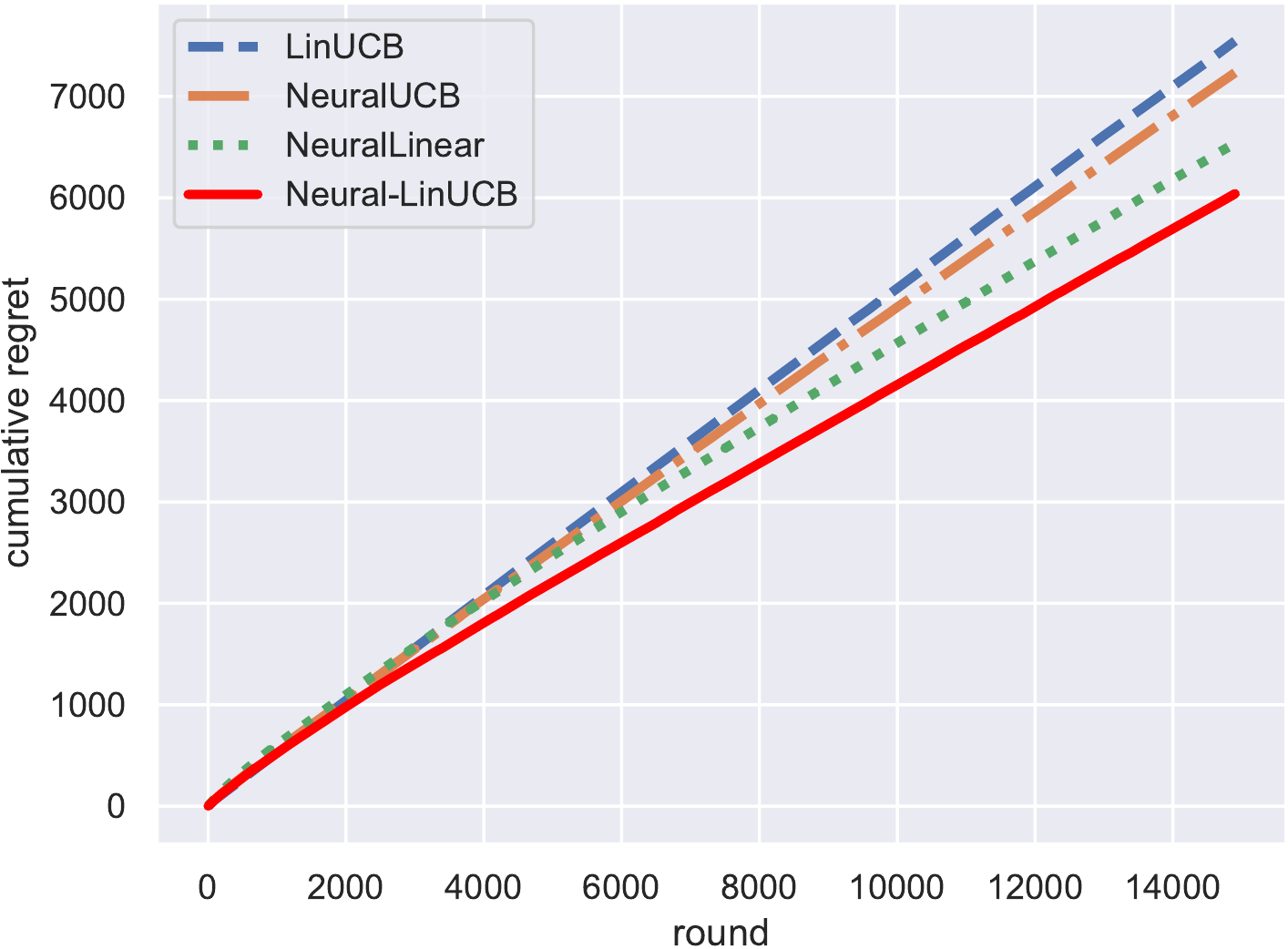}}
\caption{The cumulative regrets of LinUCB, NeuralUCB, Neural-Linear and $\algname$ over $15,000$ rounds. Experiments are averaged over 10 repetitions. 
} 
\label{fig:k10}
\end{figure*}

\noindent\textbf{Implementations:} for LinUCB, we follow the setting in \citet{li2010contextual} to use disjoint models for different arms. For NeuralUCB, Neural-Linear and $\algname$, we use a ReLU neural network defined as in \eqref{eq:def_network}, where we set the width $m=2000$ and the hidden layer length $\layerLen=2$. We set the time horizon $T=15,000$, which is the total number of rounds for each algorithm on each dataset. We use gradient decent to optimize the network weights, with a step size $\eta_{\epochNum}=$1e-5 and maximum iteration number $n=1,000$. To speed up the training process, the network parameter $\nnweight$ is updated every $\bufferLen=100$ rounds. We also apply early stopping when the loss difference of two consecutive iterations is smaller than a threshold of 1e-6. We set $\lambda=1$ and $\alpha_t=0.02$ for all algorithms, $t\in[T]$.
Following the setting in \citet{riquelme2018deep}, we use round-robin to independently select each arm for $3$ times at the beginning of each algorithm. For NeuralUCB, since it is computationally unaffordable to perform the original UCB exploration as displayed in \cite{zhou2019neural}, we follow their experimental setting to replace the matrix $\Zb_t\in\RR^{(d+p)\times (d+p)}$ in \citet{zhou2019neural} with its diagonal matrix.  


\noindent\textbf{Results:} we plot the cumulative regret of all algorithms versus round in Figures \ref{fig:statlog}, \ref{fig:magic} and \ref{fig:covertype}. The results are reported based on the average of 10 repetitions over different random shuffles of the datasets. It can be seen that  algorithms based on neural network representations (NeuralUCB, Neural-Linear and $\algname$) consistently outperform the linear contextual bandit method LinUCB, which shows that linear models may lack representation power and find biased estimates for the underlying reward generating function. Furthermore,  our proposed $\algname$  consistently achieves a significantly lower cumulative regret than the NeuralUCB algorithm. The possible reasons for this improvement are two-fold: (1) the replacement of the feature matrix with its diagonal matrix in NeuralUCB causes the UCB exploration to be biased and not accurate enough; (2) the exploration over the whole weight space of a deep neural network may potentially lead to spurious local minima that generalize  poorly to unseen data. In addition, our algorithm also achieves a lower regret than Neural-Linear on the tested datasets. 

The results in our experiment are well aligned with our theory that deep representation and shallow exploration are sufficient to guarantee a good performance of neural contextual bandit algorithms, which is also consistent with the findings in existing literature \citep{riquelme2018deep} that decoupling the representation learning and uncertainty estimation improves the performance. Moreover, we also find that our $\algname$ algorithm is much more computationally efficient than NeuralUCB since we only perform the UCB exploration on the last layer of the neural network. In specific, on the \emph{Statlog} dataset, it  takes $19,028$ seconds for NeuralUCB to finish $15,000$ rounds and achieve the regret in Figure \ref{fig:statlog}, while it only takes $783$ seconds for $\algname$. On the \emph{Magic} dataset, the runtimes of NeuralUCB and $\algname$ are $18,579$ seconds and $7,263$ seconds respectively. And on the \emph{Covertype} dataset, the runtimes of NeuralUCB and $\algname$ are $11,941$ seconds and $3,443$ seconds respectively. For practical applications in the real-world with larger problem sizes, we believe that the improvement of our algorithm in terms of the computational efficiency will be more significant.




\section{Conclusions}
In this paper, we propose a new neural contextual bandit algorithm called $\algname$, which uses the hidden layers of a ReLU neural network as a deep representation of the raw feature vectors and performs UCB type exploration on the last layer of the neural network. By incorporating techniques in liner contextual bandits and neural tangent kernels, we prove that the proposed algorithm achieves a sublinear regret when the width of the network is sufficiently large. This is the first regret analysis of neural contextual bandit algorithms with deep representation and shallow exploration, which have been observed in practice to work well on many benchmark bandit problems \citep{riquelme2018deep}. We also conducted experiments on real-world datasets to demonstrate the advantage of the proposed algorithm over linear contextual bandits and existing neural contextual bandit algorithms.

\appendix

\section{Proof of the Main Theory}\label{sec:proof_main}
In this section, we provide the proof of the regret bound for $\algname$. Recall that in neural contextual bandits, we do not assume a specific formulation of the underlying reward generating function $r(\cdot)$. Instead, we use deep neural networks defined in Section \ref{sec:neural_bandit} to approximate $r(\cdot)$. We will first show that the reward generating function $r(\cdot)$ can be approximated by the local linearization of the overparameterized neural network near the initialization weight $\nnweight^{(0)}$. In particular, we denote the gradient of $\bphi(\xb;\nnweight)$ with respect to $\nnweight$ by $\gb(\xb;\nnweight)$, namely,
\begin{align}\label{eq:def_grad_phi}
    \gb(\xb;\nnweight)=\nabla_{\nnweight} \bphi(\xb;\nnweight),
\end{align}
which is a matrix in $\RR^{d\times p}$. We define $\phi_j(\xb;\nnweight)$ to be the $j$-th entry of vector $\bphi(\xb;\nnweight)$, for any $j\in[d]$. Then, we can prove the following lemma.
\begin{lemma}\label{lemma:linearization}
Suppose Assumptions \ref{asp:ntk_pd} hold. Then there exists  $\nnweight^*\in\RR^{p}$ such that $\|\nnweight^*-\nnweight^{(0)}\|_2\leq 1/\sqrt{ m}\sqrt{(\rb-\tilde\rb)^{\top}\Hb^{-1}(\rb-\tilde\rb)}$ and it holds that
\begin{align*}
    r(\xb_{t,k})
    &=\btheta^{*\top}\bphi(\xb_{t,k};\nnweight_{t-1})+\btheta_0^{\top}\gb(\xb_{t,k};\nnweight^{(0)})\big(\nnweight^{*}-\nnweight^{(0)}\big),
\end{align*}
for all $k\in[K]$ and $t=1,\ldots,T$. 
\end{lemma}
Lemma \ref{lemma:linearization} implies that the reward generating function $r(\cdot)$ at points $\{\xb_{i,k}\}_{i\in[T],k\in[K]}$ can be approximated by a linear function around the initial point $\nnweight^{(0)}$. Note that a similar lemma is also proved in \citet{zhou2019neural} for NeuralUCB.

The next lemma shows the upper bounds of the output of the neural network $\bphi$ and its gradient.
\begin{lemma}\label{lemma:gradient_function_bound}
Suppose Assumptions \ref{asp:nondegen} and \ref{asp:ntk_pd} hold. For any round index $t\in[T]$, suppose it is in the $\epochNum$-th epoch of Algorithm \ref{alg:update_nn}, i.e.,   $t=(\epochNum-1)\bufferLen+i$ for some $i\in[\bufferLen]$. If the step size $\eta_{\epochNum}$ in Algorithm \ref{alg:update_nn}  satisfies 
\begin{align*}
    \eta \leq \frac{C_{0}}{d^2  m\nnIter T^{5.5}\layerLen^{6}\log(TK/\delta)},
\end{align*}
and the width of the neural network satisfies
\begin{align}
    m\geq\max\{\layerLen\log(TK/\delta), d\layerLen^2\log(m/\delta),\delta^{-6}\bufferLen^{18}\layerLen^{16}\log^3(TK)\},
\end{align}
then, with probability at least $1-\delta$ we have
\begin{align*}
    \|\nnweight_t-\nnweight^{(0)}\|_{2}&\leq \frac{\delta^{3/2}}{m^{1/2}Tn^{9/2}\layerLen^6\log^3(m)},\\
    \|\gb(\xb_{t,k};\nnweight^{(0)})\|_F&\leq C_{1}\sqrt{d\layerLen m},\\
    \|\bphi(\xb;\nnweight_t)\|_2&\leq \sqrt{d\log(\nnIter)\log(TK/\delta)},
\end{align*}
for all $t\in[T]$, $k\in[K]$, where the neural network $\bphi$ is defined in \eqref{eq:def_hidden_layer} and its gradient is defined in~\eqref{eq:def_grad_phi}.
\end{lemma}

The next lemma shows that the neural network $\bphi(\xb;\nnweight)$ is close to a linear function in terms of the weight $\nnweight$ parameter around a small neighborhood of the initialization point $\nnweight^{(0)}$.
\begin{lemma}[Theorems 5 in \cite{cao2019generalization2}]\label{lemma:local_linear}
Let $\nnweight,\nnweight'$ be  in the neighborhood of $\nnweight_0$, i.e., $\nnweight,\nnweight'\in\BB(\nnweight_0,\omega)$ for some $\omega>0$. Consider the neural network defined in \eqref{eq:def_hidden_layer}, if the width $m$ and the radius $\omega$ of the neighborhood satisfy
\begin{align*}
    m&\geq C_0\max\{d\layerLen^2\log(m/\delta),\omega^{-4/3}\layerLen^{-8/3}\log(TK)\log(m/(\omega\delta))\},\\
    \omega&\leq C_1\layerLen^{-5}(\log m)^{-3/2},
\end{align*}
then for all $\xb\in\{\xb_{t,k}\}_{t\in[T],k\in[K]}$, with probability at least $1-\delta$ it holds that
\begin{align*}
    &|\phi_j(\xb;\nnweight)-\hat \phi_j(\xb;\nnweight)|
    \leq C_2\omega^{4/3}L^{3}d^{-1/2}\sqrt{m\log m},
\end{align*}
where $\hat\phi_j(\xb;\nnweight)$ is the linearization of $\phi_j(\xb;\nnweight)$ at $\nnweight'$ defined as follow:
\begin{align}\label{eq:def_linear_approx_phi}
    \hat\phi_j(\xb;\nnweight)=\phi_j(\xb;\nnweight')+\la\nabla_{\nnweight}\phi_j(\xb;\nnweight'),\nnweight-\nnweight'\ra.
\end{align}
\end{lemma}
Similar results on the local linearization of an overparameterized neural network are also presented in \cite{allen2019convergence,cao2019generalization2}.

For the output layer $\btheta^*$, we perform a UCB type exploration and thus we need to characterize the uncertainty of the estimation. The next lemma shows the confidence bound of the estimate $\btheta_t$ in Algorithm \ref{alg:deepUCB}.
\begin{lemma}\label{lemma:confidence_bound}
Suppose Assumption and \ref{asp:ntk_pd} hold. For any $\delta\in(0,1)$, with probability at least $1-\delta$, the distance between the estimated weight vector $\btheta_t$ by Algorithm \ref{alg:deepUCB}  and $\btheta^*$  can be bounded as follows:
\begin{align*}
  &\bigg\| \btheta_{t}-\btheta^*-\Ab_t^{-1}\sum_{s=1}^t\bphi(\xb_{s,a_s};\nnweight_{s-1})\btheta_0^{\top}\gb(\xb_{s,a_s};\nnweight^{(0)})(\nnweight^*-\nnweight^{(0)})\bigg\|_{\Ab_t}\\
   &\leq \nu\sqrt{2\big(d\log(1+t(\log \bufferLen K)/\lambda )+\log 1/\delta\big)}+\lambda^{1/2}M,
\end{align*}
for any $t\in[T]$.
\end{lemma}
Note that the confidence bound in Lemma \ref{lemma:confidence_bound} is different from the standard result for linear contextual bandits in \cite{abbasi2011improved}. The additional term on the left hand side of the confidence bound is due to the bias caused by the representation learning using a deep neural network. To deal with this extra term, we need the following technical lemma.
\begin{lemma}\label{lemma:A_inverse_sum_phi}
Assume  that $\Ab_t=\lambda\Ib+\sum_{s=1}^t\bphi_s\bphi_s^{\top}$, where $\bphi_t\in\RR^d$ and $\|\bphi_t\|_2\leq G$ for all $t\geq 1$ and some constants $\lambda,G>0$. Let $\{\zeta_t\}_{t=1,\ldots}$ be a real-value sequence such that $|\zeta_t|\leq U$ for some constant $U>0$. Then we have
\begin{align*}
   \bigg\|\Ab_t^{-1}\sum_{s=1}^t \bphi_s \zeta_s\bigg\|_2\leq 2Ud, \quad \forall t=1,2,\ldots
\end{align*}
\end{lemma}
The next lemma provides some standard bounds on the feature matrix $\Ab_t$, which is a combination of Lemma 10 and Lemma 11 in \cite{abbasi2011improved}.
\begin{lemma}\label{lemma:det_sum}
Let $\{\xb_t\}_{t=1}^{\infty}$ be a sequence in $\RR^d$ and $\lambda>0$. Suppose $\|\xb_t\|_2\leq G$ and $\lambda\geq\max\{1,G^2\}$ for some $G>0$. Let $\Ab_t=\lambda\Ib+\sum_{s=1}^{t}\xb_t\xb_t^{\top}$. Then we have
\begin{align*}
    \det(\Ab_t)\leq(\lambda+tG^2/d)^d,\quad\text{and }\sum_{t=1}^{T}\|\xb_t\|_{\Ab_{t-1}^{-1}}^2\leq2\log\frac{\det(\Ab_T)}{\det(\lambda\Ib)}\leq 2d\log(1+TG^2/(\lambda d)).
\end{align*}
\end{lemma}

Now we are ready to prove the regret bound of Algorithm \ref{alg:deepUCB}.
\begin{proof}[Proof of Theorem \ref{thm:regret}]
For a time horizon $T$, without loss of generality, we assume $T=Q\bufferLen$ for some epoch number $Q$. By the definition of regret in \eqref{eq:def_regret}, we have
\begin{align*}
    R_T&=\EE\bigg[\sum_{t=1}^{T}(\empR(\xb_{t,a_t^*})-\empR(\xb_{t,a_t}))\bigg]=\EE\bigg[\sum_{q=1}^{Q}\sum_{i=1}^{\bufferLen}(\empR(\xb_{q\bufferLen+i,a_{q\bufferLen+i}^*})-\empR(\xb_{q\bufferLen+i,a_{q\bufferLen+i}}))\bigg].
\end{align*}
Note that for the simplicity of presentation, we omit the conditional expectation notation of $\nnweight^{(0)}$ in the rest of the proof when the context is clear. In the second equation, we rewrite the time index $t=q\bufferLen+i$ as the $i$-th iteration in the $q$-th epoch. 

By the definition in \eqref{eq:nn_model}, we have $\EE[\empR(\xb_{t,k})|\xb_{t,k}]=r(\xb_{t,k})$ for all $t\in[T]$ and $k\in{K}$. Based on the linearization of reward generating function, we can decompose the instaneous regret into different parts and upper bound them individually. In particular, by Lemma \ref{lemma:linearization}, there exists a vector $\nnweight^*\in\RR^p$ such that we can write the expectation of the reward generating function  as  a linear function. Then it holds that
\begin{align}\label{eq:regret_decop_0}
    r(\xb_{t,a_t^*})-r(\xb_{t,a_t})
    &=\btheta_0^{\top}\big[\gb\big(\xb_{t,a_t^*};\nnweight^{(0)}\big)-\gb\big(\xb_{t,a_t};\nnweight^{(0)}\big)\big]\big(\nnweight^*-\nnweight^{(0)}\big)\notag\\
    &\quad+\btheta^{*\top}\big[\bphi\big(\xb_{t,a_t^*};\nnweight_{t-1}\big)-\bphi\big(\xb_{t,a_t};\nnweight_{t-1}\big)\big]\notag\\
    &=\btheta_0^{\top}\big[\gb\big(\xb_{t,a_t^*};\nnweight^{(0)}\big)-\gb\big(\xb_{t,a_t};\nnweight^{(0)}\big)\big]\big(\nnweight^*-\nnweight^{(0)}\big)\notag\\
    &\quad+\btheta_{t-1}^{\top}\big[\bphi\big(\xb_{t,a_t^*};\nnweight_{t-1}\big)-\bphi\big(\xb_{t,a_t};\nnweight_{t-1}\big)\big]\notag\\
    &\quad-(\btheta_{t-1}-\btheta^*)^{\top}\big[\bphi\big(\xb_{t,a_t^*};\nnweight_{t-1}\big)-\bphi\big(\xb_{t,a_t};\nnweight_{t-1}\big)\big].
\end{align} 
The first term in \eqref{eq:regret_decop_0} can be easily bounded using the first order stability in Assumption \ref{asp:gradient_fos} and the distance between $\nnweight^*$ and $\nnweight^{(0)}$ in Lemma \ref{lemma:linearization}. The second term in \eqref{eq:regret_decop_0} is related to the optimistic rule of choosing arms in Line \ref{algline:ucb_rule} of Algorithm \ref{alg:deepUCB}, which can be bounded using the same technique for LinUCB \citep{abbasi2011improved}. For the last term in \eqref{eq:regret_decop_0}, we need to prove that the estimate of weight parameter $\btheta_{t-1}$ lies in a confidence ball centered at $\btheta^*$. 
For the ease of notation, we define 
\begin{align}\label{eq:def_mt}
    \Mb_t=\Ab_t^{-1}\sum_{s=1}^t\bphi(\xb_{s,a_s};\nnweight_{s-1})\btheta_0^{\top}\gb(\xb_{s,a_s};\nnweight^{(0)})(\nnweight^*-\nnweight^{(0)}).
\end{align}
Then the second term in \eqref{eq:regret_decop_0} can be bounded in the following way:
\begin{align}\label{eq:intro_confidence_bound}
    &-(\btheta_{t-1}-\btheta^*)^{\top}\big[\bphi\big(\xb_{t,a_t^*};\nnweight_{t-1}\big)-\bphi\big(\xb_{t,a_t};\nnweight_{t-1}\big)\big]\notag\\
    &=-\big(\btheta_{t-1}-\btheta^*-\Mb_{t-1}\big)^{\top}\bphi\big(\xb_{t,a_t^*};\nnweight_{t-1}\big)+\big(\btheta_{t-1}-\btheta^*-\Mb_{t-1}\big)^{\top}\bphi\big(\xb_{t,a_t};\nnweight_{t-1}\big)\notag\\
    &\qquad-\Mb_{t-1}^{\top}\big[\bphi\big(\xb_{t,a_t^*};\nnweight_{t-1}\big)-\bphi\big(\xb_{t,a_t};\nnweight_{t-1}\big)\big]\notag\\
    &\leq\|\btheta_{t-1}-\btheta^*-\Mb_{t-1}\|_{\Ab_{t-1}}\cdot\|\bphi(\xb_{t,a_t^*};\nnweight_{t-1})\|_{\Ab_{t-1}^{-1}}+\|\btheta_{t-1}-\btheta^*-\Mb_{t-1}\|_{\Ab_{t-1}}\cdot\|\bphi(\xb_{t,a_t};\nnweight_{t-1})\|_{\Ab_{t-1}^{-1}}\notag\\
    &\qquad+\big\|\Mb_{t-1}^{\top}\big[\bphi\big(\xb_{t,a_t^*};\nnweight_{t-1}\big)-\bphi\big(\xb_{t,a_t};\nnweight_{t-1}\big)\big]\big\|_2\notag\\
    &\leq \alpha_t\|\bphi(\xb_{t,a_t^*};\nnweight_{t-1})\|_{\Ab_{t-1}^{-1}}+\alpha_t\|\bphi(\xb_{t,a_t};\nnweight_{t-1})\|_{\Ab_{t-1}^{-1}}\notag\\
    &\qquad+\|\Mb_{t-1}\|_2\cdot\|\bphi\big(\xb_{t,a_t^*};\nnweight_{t-1}\big)-\bphi\big(\xb_{t,a_t};\nnweight_{t-1}\big)\|_2.
\end{align}
where the last inequality is due to Lemma \ref{lemma:confidence_bound} and the choice of $\alpha_t$. Plugging \eqref{eq:intro_confidence_bound} back into \eqref{eq:regret_decop_0} yields
\begin{align} 
    r(\xb_{t,a_t^*})-r(\xb_{t,a_t})&\leq \alpha_t\|\bphi(\xb_{t,a_t};\nnweight_{t-1})\|_{\Ab_{t-1}^{-1}}-\alpha_t\|\bphi(\xb_{t,a_t^*};\nnweight_{t-1})\|_{\Ab_{t-1}^{-1}}\notag\\
    &\qquad+\alpha_t\|\bphi(\xb_{t,a_t^*};\nnweight_{t-1})\|_{\Ab_{t-1}^{-1}}+\alpha_t\|\bphi(\xb_{t,a_t};\nnweight_{t-1})\|_{\Ab_{t-1}^{-1}}\notag\\
    &\qquad+\|\Mb_{t-1}\|_2\cdot\|\bphi\big(\xb_{t,a_t^*};\nnweight_{t-1}\big)-\bphi\big(\xb_{t,a_t};\nnweight_{t-1}\big)\|_2\notag\\    
    &\qquad+\|\btheta_0\|_2\cdot\|\gb(\xb_{t,a_t^*};\nnweight^{(0)})-\gb(\xb_{t,a_t};\nnweight^{(0)})\|_F\cdot\|\nnweight^{*}-\nnweight^{(0)}\|_2\notag\\
    &\leq 2\alpha_t\|\bphi(\xb_{t,a_t};\nnweight_{t-1})\|_{\Ab_{t-1}^{-1}}+\|\Mb_{t-1}\|_2\cdot\|\bphi\big(\xb_{t,a_t^*};\nnweight_{t-1}\big)-\bphi\big(\xb_{t,a_t};\nnweight_{t-1}\big)\|_2\notag\\ 
    &\qquad+\fosPara\|\btheta_0\|_2\cdot\|\xb_{t,a_t^*}-\xb_{t,a_t}\|_2\cdot\|\nnweight^{*}-\nnweight^{(0)}\|_2,
\end{align}
where in the first inequality we used the definition of upper confidence bound in Algorithm \ref{alg:deepUCB} and the second inequality is due to Assumption \ref{asp:gradient_fos}. Recall the linearization of $\phi_j$ in Lemma \ref{lemma:local_linear}, we have
\begin{align*}
   \hat\bphi(\xb;\nnweight_{t-1})= \bphi(\xb;\nnweight_{0})+\gb(\xb;\nnweight_{0})(\nnweight_{t-1}-\nnweight_{0}).
\end{align*} 
Note that by the initialization, we have $\bphi(\xb;\nnweight_0)=\zero$ for any $\xb\in\RR^d$. Thus, it holds that
\begin{align}
    \bphi\big(\xb_{t,a_t^*};\nnweight_{t-1}\big)-\bphi\big(\xb_{t,a_t};\nnweight_{t-1}\big)&=\bphi\big(\xb_{t,a_t^*};\nnweight_{t-1}\big)-\bphi\big(\xb_{t,a_t^*};\nnweight_{0}\big)+\bphi\big(\xb_{t,a_t};\nnweight_{0}\big)-\bphi\big(\xb_{t,a_t};\nnweight_{t-1}\big)\notag\\
    &=\bphi\big(\xb_{t,a_t^*};\nnweight_{t-1}\big)-\hat\bphi\big(\xb_{t,a_t^*};\nnweight_{t-1}\big)+\gb(\xb_{t,a_t^*};\nnweight_{0})(\nnweight_{t-1}-\nnweight_{0})\notag\\
    &\qquad+\bphi\big(\xb_{t,a_t};\nnweight_{t-1}\big)-\bphi\big(\xb_{t,a_t};\nnweight_{t-1}\big)-\gb(\xb_{t,a_t};\nnweight_{0})(\nnweight_{t-1}-\nnweight_{0}),
\end{align}
which immediately implies that
\begin{align}
    &\big\|\bphi\big(\xb_{t,a_t^*};\nnweight_{t-1}\big)-\bphi\big(\xb_{t,a_t};\nnweight_{t-1}\big)\big\|_2\notag\\
    &\leq\big\|\bphi\big(\xb_{t,a_t^*};\nnweight_{t-1}\big)-\hat\bphi\big(\xb_{t,a_t^*};\nnweight_{t-1}\big)\big\|_2+\big\|\bphi\big(\xb_{t,a_t^*};\nnweight_{t-1}\big)-\hat\bphi\big(\xb_{t,a_t^*};\nnweight_{t-1}\big)\big\|_2\notag\\
    &\qquad+\big\|\big(\gb(\xb_{t,a_t^*};\nnweight_{0})-\gb(\xb_{t,a_t};\nnweight_{0})\big)(\nnweight_{t-1}-\nnweight_{0})\big\|_2\notag\\
    &\leq C_0\omega^{4/3}L^3d^{1/2}\sqrt{m\log m}+\fosPara\|\xb_{t,a_t^*}-\xb_{t,a_t}\|_2\|\nnweight_{t-1}-\nnweight^{(0)}\|_2,
\end{align}
where the second inequality is due to Lemma \ref{lemma:local_linear} and Assumption \ref{asp:gradient_fos}. Therefore, the instaneous regret can be further upper bounded as follows.
\begin{align}\label{eq:decomp_instaneous_regret}   
  r(\xb_{t,a_t^*})-r(\xb_{t,a_t})&\leq   2\alpha_t\|\bphi(\xb_{t,a_t};\nnweight_{t-1})\|_{\Ab_{t-1}^{-1}}+\fosPara\|\btheta_0\|_2\cdot\|\xb_{t,a_t^*}-\xb_{t,a_t}\|_2\cdot\|\nnweight^{*}-\nnweight^{(0)}\|_2\notag\\
  &\qquad+\|\Mb_{t-1}\|_2\cdot\big(C_0\omega^{4/3}L^3d^{1/2}\sqrt{m\log m}+\fosPara\|\xb_{t,a_t^*}-\xb_{t,a_t}\|_2\|\nnweight_{t-1}-\nnweight^{(0)}\|_2\big).
\end{align}
By Assumption \ref{asp:nondegen} we have $\|\xb_{t,a_t^*}-\xb_{t,a_t}\|_2\leq 2$. By Lemma \ref{lemma:linearization} and Lemma \ref{lemma:gradient_function_bound}, we have
\begin{align}
    \|\nnweight^{*}-\nnweight^{(0)}\|_2\leq\sqrt{1/m(\rb-\tilde\rb)^{\top}\Hb^{-1}(\rb-\tilde\rb)}, \quad \|\nnweight_{t}-\nnweight^{(0)}\|_2\leq\frac{\delta^{3/2}}{m^{1/2}T\nnIter^{9/2}\layerLen^6\log^3(m)}.
\end{align}
In addition, since the entries of $\btheta_0$ are i.i.d. generated  from $N(0,1/d)$, we have $\|\btheta_0\|_2\leq 2(2+\sqrt{d^{-1}\log(1/\delta)})$ with probability at least $1-\delta$ for any $\delta>0$. By Lemma \ref{lemma:gradient_function_bound}, we have 
$\|\gb(\xb_{t,a_t};\nnweight^{(0)})\|_F\leq C_{1}\sqrt{dm}$. Therefore, 
\begin{align*}
    \big|\btheta_0^{\top}\gb(\xb_{s,a_s};\nnweight^{(0)})(\nnweight^*-\nnweight^{(0)})\big|\leq C_{2}d\sqrt{\log(1/\delta)(\rb-\tilde\rb)^{\top}\Hb^{-1}(\rb-\tilde\rb)}.
\end{align*}
Then, by the definition of $\Mb_t$ in \eqref{eq:def_mt} and Lemma \ref{lemma:A_inverse_sum_phi}, we have
\begin{align}\label{eq:bound_mt}
    \|\Mb_{t-1}\|_2\leq C_{3}d^2\sqrt{\log(1/\delta)(\rb-\tilde\rb)^{\top}\Hb^{-1}(\rb-\tilde\rb)}.
\end{align}
Substituting \eqref{eq:bound_mt} and the above results on $\|\xb_{t,a_t}-\xb_{t,a_t^*}\|_2$, $\|\btheta_0\|_2$, $\|\nnweight^*-\nnweight^{(0)}\|_2$ and $\|\nnweight_{t-1}-\nnweight^{(0)}\|_2$ back into \eqref{eq:decomp_instaneous_regret} further yields
\begin{align*}
   &r(\xb_{t,a_t^*})-r(\xb_{t,a_t})\\
   &\leq 2\alpha_t\|\bphi(\xb_{t,a_t};\nnweight_{t-1})\|_{\Ab_{t-1}^{-1}}+C_4\fosPara m^{-1/2}\sqrt{\log(1/\delta)(\rb-\tilde\rb)^{\top}\Hb^{-1}(\rb-\tilde\rb)}\\
   &\qquad+\bigg(C_0\omega^{4/3}L^3d^{1/2}\sqrt{m\log m}+\frac{2\fosPara\delta^{3/2}}{m^{1/2}T\nnIter^{9/2}\layerLen^6\log^3(m)}\bigg)C_{3}d^2\sqrt{\log(1/\delta)(\rb-\tilde\rb)^{\top}\Hb^{-1}(\rb-\tilde\rb)}.
\end{align*}
Note that we have $\omega=O(m^{-1/2}\|\rb-\tilde\rb\|_{\Hb^{-1}})$ by Lemma \ref{lemma:linearization}.
Therefore, the regret of the $\algname$ is
\begin{align*}
    R_T&\leq\sqrt{Q\bufferLen\max_{t\in[T]}\alpha_t^2\sum_{q=1}^{Q}\sum_{i=1}^{\bufferLen}\|\bphi(\xb_{i,a_i};\nnweight_{q\bufferLen+i})\|_{\Ab_{i}^{-1}}^2}+C_4\fosPara m^{-1/2}T\sqrt{\log(1/\delta)}\|\rb-\tilde\rb\|_{\Hb^{-1}}\\
    &\qquad+\bigg(\frac{C_0T\layerLen^3 d^{1/2}\sqrt{\log m} \|\rb-\tilde\rb\|_{\Hb^{-1}}^{4/3}}{m^{1/6}}+\frac{2\fosPara\delta^{3/2}}{m^{1/2}\nnIter^{9/2}\layerLen^6\log^3(m)}\bigg)C_{3}d^2\sqrt{\log(1/\delta)}\|\rb-\tilde\rb\|_{\Hb^{-1}}\\
    &\leq C_5\sqrt{Td\log(1+ TG^2/(\lambda d))}\big(\nu\sqrt{d\log(1+T(\log T K)/\lambda)+\log1/\delta}+\lambda^{1/2}M\big)\\
    &\qquad+C_6\fosPara \layerLen^3 d^{5/2}m^{-1/6}T\sqrt{\log m\log(1/\delta)\log(TK/\delta)}\|\rb-\tilde\rb\|_{\Hb^{-1}},
\end{align*}
where the first inequality is due to Cauchy's inequality, the second inequality comes from the upper bound of $\alpha_t$ in Lemma \ref{lemma:confidence_bound} and Lemma \ref{lemma:det_sum}. $\{C_j\}_{j=0,\ldots,6}$ are absolute constants that are independent of problem parameters. 
\end{proof}

\section{Proof of Technical Lemmas}
In this section, we provide the proof of technical lemmas used in the regret analysis of Algorithm \ref{alg:deepUCB}.

\subsection{Proof of Lemma \ref{lemma:linearization}}
Before we prove the lemma, we first present some notations and a supporting lemma for simplification. Let $\bbeta=(\btheta^{\top},\nnweight^{\top})^{\top}\in\RR^{d+p}$ be the concatenation of the exploration parameter and the hidden layer parameter of the neural network $f(\xb;\bbeta)=\btheta^{\top}\bphi(\xb;\nnweight)$. Note that for any input data vector $\xb\in\RR^d$, we have
\begin{align}\label{eq:gradient_beta}
    \frac{\partial}{\partial\bbeta}f(\xb;\bbeta)= \bigg(\bphi(\xb;\nnweight)^{\top},\btheta^{\top}\frac{\partial}{\partial\nnweight}\bphi(\xb;\nnweight)\bigg)^{\top}=\big(\bphi(\xb;\nnweight)^{\top},\btheta^{\top}\gb(\xb;\nnweight)\big)^{\top},
\end{align}
where $\gb(\xb;\nnweight)$ is the partial gradient of $\bphi(\xb;\nnweight)$ with respect to $\nnweight$ defined in \eqref{eq:def_grad_phi}, which is a matrix in $\RR^{d\times p}$. Similar to \eqref{eq:def_ntk}, we define $\Hb_{\layerLen+1}$ to be the neural tangent kernel matrix based on all $\layerLen+1$ layers of the neural network $f(\xb;\bbeta)$. Note that by the definition of $\Hb$ in \eqref{eq:def_ntk}, we must have $\Hb_{\layerLen+1}=\Hb+\Bb$ for some positive definite matrix $\Bb\in\RR^{TK\times TK}$. The following lemma shows that the NTK matrix is close to the matrix defined based on the gradients of the neural network on $TK$ data points.
\begin{lemma}[Theorem 3.1 in \citet{arora2019exact}]\label{lemma:NTK_gradient}
Let $\epsilon>0$ and $\delta\in(0,1)$. Suppose the activation function in \eqref{eq:def_network} is ReLU, i.e., $\sigma_l(x)=\max(0,x)$, and the width of the neural network satisfies
\begin{align}\label{eq:condition_width_ntk}
    m\geq\bOmega\bigg(\frac{\layerLen^{14}}{\epsilon^4}\log \bigg(\frac{\layerLen}{\delta}\bigg)\bigg).
\end{align}
Then for any $\xb,\xb'\in\RR^d$ with $\|\xb\|_2=\|\xb'\|_2=1$, with probability at least $1-\delta$ over the randomness of the initialization of the network weight $\nnweight$ it holds that
\begin{align*}
    \bigg|\bigg\la\frac{1}{\sqrt{m}}\frac{\partial f(\bbeta,\xb)}{\partial\bbeta},\frac{1}{\sqrt{m}}\frac{\partial f(\bbeta,\xb')}{\partial\bbeta}\bigg\ra-\Hb_{\layerLen+1}(\xb,\xb')\bigg|\leq\epsilon.
\end{align*}
\end{lemma}
Note that in the above lemma, there is a factor $1/\sqrt{m}$ before the gradient. This is due to the additional $\sqrt{m}$ factor in the definition of the neural network in \eqref{eq:def_network}, which ensures the value of the neural network function evaluated at the initialization is of the order $O(1)$.
\begin{proof}[Proof of Lemma \ref{lemma:linearization}]
Recall that we renumbered the feature vectors $\{\xb_{t,k}\}_{t\in[T],k\in[K]}$ for all arms from round 1 to round $T$ as $\{\xb_i\}_{i=1,\ldots,TK}$. By concatenating the gradients at different inputs and the gradient in \eqref{eq:gradient_beta}, we  define  $\bPsi\in\RR^{T K\times (d+p)}$ as follows.
\begin{align*}
    \bPsi=\frac{1}{\sqrt{m}}\begin{bmatrix}
    \frac{\partial}{\partial\bbeta}\btheta^{\top}\bphi(\xb_{1};\nnweight)\\
    \vdots\\
    \frac{\partial}{\partial\bbeta}\btheta^{\top}\bphi(\xb_{TK};\nnweight)
    \end{bmatrix}
    =\frac{1}{\sqrt{m}}\begin{bmatrix}
     \bphi(\xb_{1};\nnweight^{(0)})^{\top}&\btheta_0^{\top}\gb(\xb_{1};\nnweight^{(0)})\\
     \vdots&\vdots\\
    \bphi(\xb_{i};\nnweight^{(0)})^{\top}&\btheta_0^{\top}\gb(\xb_{i};\nnweight^{(0)})\\
     \vdots&\vdots\\
     \bphi(\xb_{TK};\nnweight^{(0)})^{\top}&\btheta_0^{\top}\gb(\xb_{TK};\nnweight^{(0)})
    \end{bmatrix}.
\end{align*} 
By Applying Lemma \ref{lemma:NTK_gradient}, we know  with probability at least $1-\delta$ it holds that
\begin{align*}
    |\la\bPsi_{j*},\bPsi_{l*}\ra-\Hb_{\layerLen+1}(\xb_{j},\xb_{l})|\leq\epsilon
\end{align*}
for any $\epsilon>0$ as long as the width $m$ satisfies the condition in \eqref{eq:condition_width_ntk}. By applying union bound over all data points $\{\xb_{1},\ldots,\xb_{t},\ldots,\xb_{TK}\}$, we further have
\begin{align*}
    \|\bPsi\bPsi^{\top}-\Hb_{\layerLen+1}\|_F\leq T K\epsilon.
\end{align*}
Note that $\Hb$ is the neural tangent kernel (NTK) matrix  defined in \eqref{eq:def_ntk} and $\Hb_{\layerLen+1}$ is the NTK matrix defined based on all $\layerLen+1$ layers. By Assumption \ref{asp:ntk_pd}, $\Hb$ has a minimum eigenvalue $\lambda_0>0$, which is defined based on the first $\layerLen$ layers of $f$. Furthermore, by the definition of NTK matrix in \eqref{eq:def_ntk}, we know that $\Hb_{\layerLen+1}=\Hb+\Bb$ for some semi-positive definite matrix $\Bb$. Therefore, the NTK matrix $\Hb_{\layerLen+1}$ defined based on all $\layerLen+1$ layers is also positive definite and its minimum eigenvalue is lower bounded by $\lambda_0$. Let $\epsilon=\lambda_0/(2T K)$. By triangle equality we have $\bPsi\bPsi^{\top}\succ\Hb_{\layerLen+1}-\|\bPsi\bPsi^{\top}-\Hb_{\layerLen+1}\|_2\Ib\succ\Hb_{\layerLen+1}-\|\bPsi\bPsi^{\top}-\Hb_{\layerLen+1}\|_F\Ib\succ \Hb_{\layerLen+1}-\lambda_0/2\Ib \succ 1/2\Hb_{\layerLen+1}$, which means that $\bPsi$ is semi-definite positive and thus $\text{rank}(\bPsi)=T K$ since $m>TK$. 

We assume that $\bPsi$ can be decomposed as $\bPsi=\Pb\Db\Qb^{\top}$, where $\Pb\in\RR^{T K\times T K}$ is the eigenvectors of $\bPsi\bPsi^{\top}$ and thus $\Pb\Pb^{\top}=\Ib_{T K}$, $\Db\in\RR^{T K\times T K}$ is a diagonal matrix with the square root of eigenvalues of $\bPsi\bPsi^{\top}$, and $\Qb^{\top}\in\RR^{T K\times (d+p)}$ is the eigenvectors of $\bPsi^{\top}\bPsi$ and thus $\Qb^{\top}\Qb=\Ib_{T K}$. We use $\Qb_1\in\RR^{d\times T K}$ and $\Qb_2\in\RR^{p\times T K}$ to denote the two blocks of $\Qb$ such that $\Qb^{\top}=[\Qb_1^{\top},\Qb_2^{\top}]$. By definition, we have 
\begin{align*}
    \Qb^{\top}\Qb=[\Qb_1^{\top},\Qb_2^{\top}]\begin{bmatrix}
    \Qb_1\\
    \Qb_2
    \end{bmatrix}
    =\Qb_1^{\top}\Qb_1+\Qb_{2}^{\top}\Qb_2=\Ib_{T K}.
\end{align*}
Note that the minimum singular value of $\Qb_1\in\RR^{d\times TK}$ is zero since $d$ is a fixed number and $TK>d$. Therefore, it must hold that $\text{rank}(\Qb_2)=TK$ and thus $\Qb_2^{\top}\Qb_2$ is positive definite. Let $\rb=(r(\xb_{1}),\ldots,r(\xb_{i}),\ldots,r(\xb_{TK}))^{\top}\in\RR^{T K}$ denote the vector of all possible rewards. We further define $\Gb\in\RR^{T K d\times p}$ and $\bPhi\in\RR^{T Kd}$  as follows
\begin{align}\label{eq:def_G_bphi}
    \Gb=\frac{1}{\sqrt{m}}\begin{bmatrix}
    \gb(\xb_{1};\nnweight^{(0)})\\
    \vdots\\
    \gb(\xb_{i};\nnweight^{(0)})\\
    \vdots\\
    \gb(\xb_{TK};\nnweight^{(0)})
    \end{bmatrix},\quad
     \bPhi=
    \begin{bmatrix}
    \bphi(\xb_{1,1};\nnweight_{0})\\
    \vdots\\
    \bphi(\xb_{t,k};\nnweight_{t-1)}\\
    \vdots\\
    \bphi(\xb_{T,K};\nnweight_{T-1})
    \end{bmatrix}.
\end{align}
and  $\bTheta,\bTheta_0\in\RR^{T K\times T K d}$ as follows
\begin{align}\label{eq:def_theta_star_zero}
    \bTheta^*=
    \begin{bmatrix}
    \btheta^{*\top}&&&\\
    &\ddots&&\\
    &&\btheta^{*\top}&&\\
    &&&\ddots&\\
    &&&&\btheta^{*\top}
    \end{bmatrix},
    \quad
    \bTheta_0=
    \begin{bmatrix}
    \btheta_0^{\top}&&&\\
    &\ddots&&\\
    &&\btheta_0^{\top}&&\\
    &&&\ddots&\\
    &&&&\btheta_{0}^{\top}
    \end{bmatrix},
\end{align}
It can be verified that $\bPsi=\Pb\Db[\Qb_1^{\top},\Qb_2^{\top}]$ and $\Pb\Db\Qb_2^{\top}=\bTheta_0\Gb$. Note that we have $\Qb_2^{\top}\Qb_2$ is positive definite by Assumption \ref{asp:ntk_pd}, which corresponds to the neural tangent kernel matrix defined on the first $\layerLen$ layers.  Then we can define $\nnweight^*$ as follows
\begin{align}
    \nnweight^*=\nnweight^{(0)}+1/\sqrt{ m}\Qb_2(\Qb_2^{\top}\Qb_2)^{-1}\Db^{-1}\Pb^{\top}(\rb-\bTheta^*\bPhi).
\end{align}
We can verify that
\begin{align*}
    \bTheta^*\bPhi+\sqrt{ m }\Pb\Db\Qb_2^{\top}(\nnweight^*-\nnweight^{(0)})=\rb.
\end{align*}
On the other hand, we have
\begin{align*}
    \|\nnweight^*-\nnweight^{(0)}\|_2^2&\leq 1/{ m}(\rb-\bTheta^*\bPhi)^{\top}\Pb\Db^{-1}(\Qb_2^{\top}\Qb_2)^{-1}\Db^{-1}\Pb^{\top}(\rb-\bTheta^*\bPhi)\\
    &\leq 1/{ m}(\rb-\bTheta^*\bPhi)^{\top}\Hb^{-1}(\rb-\bTheta^*\bPhi),
\end{align*}
which completes the proof.
\end{proof}


\subsection{Proof of Lemma \ref{lemma:gradient_function_bound}}
Note that we can view the output of the last hidden layer $\bphi(\xb;\nnweight)$ defined in \eqref{eq:def_hidden_layer} as a vector-output neural network with weight parameter $\nnweight$. 
The following lemma shows that the output of the neural network $\bphi$ is bounded at the initialization. 
\begin{lemma}[Lemma 4.4 in \cite{cao2019generalization2}]\label{lemma:output_nn_bound}
Let $\delta\in(0,1)$, and the width of the neural network satisfy  $m\geq C_0\layerLen\log(T K\layerLen/\delta)$. Then for all $t\in[T]$, $k\in[K]$ and $j\in[d]$, we have $|\phi_j(\xb_{t,k};\nnweight^{(0)})|\leq C_1\sqrt{\log(TK/\delta)}$ with probability at least $1-\delta$, where $\nnweight^{(0)}$ is the initialization of the neural network. 
\end{lemma}
In addition, in a smaller neighborhood of the initialization, the gradient of the neural network $\bphi$ is uniformly bounded.
\begin{lemma}[Lemma B.3 in \cite{cao2019generalization2}]\label{lemma:bounded_gradient}
Let $\omega\leq C_0\layerLen^{-6}(\log m)^{-3}$ and $\nnweight\in\BB(\nnweight_0,\omega)$. 
Then for all $t\in[T]$, $k\in[K]$ and $j\in[d]$, the gradient of the neural network $\bphi$ defined in \eqref{eq:def_hidden_layer} satisfies $\|\nabla_{\nnweight} \phi_j(\xb_{t,k};\nnweight)\|_2\leq C_1\sqrt{\layerLen m}$ with probability at least $1-TK\layerLen^2\exp(-C_2 m\omega^{2/3}\layerLen)$. 
\end{lemma}
The next lemma provides an upper bound on the gradient of the squared loss function defined in \eqref{eq:def_loss_nn}. Note that our definition of the loss function is slightly different from that in \cite{allen2019convergence} due to the output layer $\btheta_i$ and thus there is an additional term on the upper bound of $\|\btheta_i\|_2$ for all $i\in[T]$.
\begin{lemma}[Theorem 3 in \cite{allen2019convergence}]\label{lemma:gradient_upper_bound}
Let $\omega\leq C_0\delta^{3/2}/(T^{9/2}\layerLen^6\log^3m)$. For all $\nnweight\in\BB(\nnweight^{(0)},\omega)$, with probability at least $1-\exp(-C_{1}m\omega^{2/3}\layerLen)$ over the randomness of $\nnweight^{(0)}$, it holds that
\begin{align*}
    \|\nabla\cL(\nnweight)\|_2^2\leq\frac{C_{2}T m\cL(\nnweight)\sup_{i=1,\ldots,\bufferLen}\|\btheta_i\|_2^2}{d}.
\end{align*}
\end{lemma}

\begin{proof}[Proof of Lemma \ref{lemma:gradient_function_bound}]
Fix the epoch number $\epochNum$ and we omit it in the subscripts in the rest of the proof when no confusion arises. Recall that $\nnweight^{(s)}$ is the $s$-th iterate in Algorithm \ref{alg:update_nn}. Let $\delta>0$ be any constant. Let $\omega$ be defined as follows.
\begin{align}\label{eq:choose_omega}
    \omega=\delta^{3/2}m^{-1/2}T^{-9/2}\layerLen^{-6}\log^{-3}(m).
\end{align}
We will prove by induction that with probability at least $1-\delta$  the following statement holds for all $s=0,1,\ldots,\nnIter$
\begin{align}\label{eq:phi_w_induction}
    \phi_j(\xb;\nnweight^{(s)})\leq C_0\sum_{h=0}^{s}\frac{ \sqrt{\log(TK/\delta)}}{h+1}, \quad\text{for }\forall j\in[d]; \text{ and } \|\nnweight_{\epochNum}^{(s)}-\nnweight^{(0)}\|\leq\omega.
\end{align}
First note that \eqref{eq:phi_w_induction} holds trivially  when $s=0$ due to Lemma \ref{lemma:output_nn_bound}. Now we assume that \eqref{eq:phi_w_induction} holds for all $j=0,\ldots,s$. The loss function in \eqref{eq:def_loss_nn} can be bounded  as follows.
\begin{align*}
    \cL(\nnweight^{(j)})=\sum_{i=1}^{\epochNum\bufferLen}(\btheta_i^{\top}\bphi(\xb_i;\nnweight^{(j)})-\hat r_i)^2\leq \sum_{i=1}^{\epochNum\bufferLen}2(\|\btheta_i\|_2^2\cdot\|\bphi(\xb_i;\nnweight^{(j)})\|_2^2+1).
\end{align*}
By the update rule of $\btheta_t$, we have
\begin{align}\label{eq:bound_theta_t}
    \|\btheta_t\|_2=\bigg\|\bigg(\lambda\Ib+\sum_{i=1}^t\bphi(\xb_i;\nnweight_{i-1})\bphi(\xb_i;\nnweight_{i-1})^{\top}\bigg)^{-1}\sum_{i=1}^{t}\bphi(\xb_i;\nnweight_{i-1})\hat\rb\bigg\|_2\leq 2d,
\end{align}
where the inequality is due to Lemma \ref{lemma:A_inverse_sum_phi}, which combined with \eqref{eq:phi_w_induction} immediately implies
\begin{align}\label{eq:upper_bound_square_loss}
    \cL(\nnweight^{(j)})\leq C_{1}Td^{3}\log(TK/\delta)\bigg(\sum_{h=0}^j \frac{1}{h+1}\bigg)^2\leq C_{1}Td^{3}\log(TK/\delta)\log^2\nnIter.
\end{align}
Substituting \eqref{eq:bound_theta_t} and \eqref{eq:upper_bound_square_loss} into the inequality in Lemma \ref{lemma:gradient_upper_bound}, we also have
\begin{align}\label{eq:bounding_gradient_wt}
    \big\|\nabla\cL\big(\nnweight^{(j)}\big)\big\|_2\leq C_{2}\sqrt{dT m\cL(\nnweight^{(j)})}\leq C_{3}d^2T\log(\nnIter)\sqrt{m\log(TK/\delta)}.
\end{align}
Now we consider $\nnweight^{(s+1)}$. By triangle inequality we have
\begin{align}\label{eq:iterate_bound_induction}
    \big\|\nnweight^{(s+1)}-\nnweight^{(0)}\big\|_2&\leq\sum_{j=0}^{s}\big\|\nnweight^{(j+1)}-\nnweight^{(j)}\big\|_2\notag\\
    &=\sum_{j=0}^{s}\eta\big\|\nabla\cL\big(\nnweight^{(j)}\big)\big\|_2\notag\\
    &\leq \sum_{j=0}^s\eta d^{2}T\log(\nnIter)\sqrt{m\log(T K/\delta)},
\end{align}
where the last inequality is due to \eqref{eq:bounding_gradient_wt}.
If we choose the step size $\eta_{\epochNum}$ in the $\epochNum$-th epoch such that
\begin{align}
    \eta\leq\frac{\omega}{d^2T\nnIter\log(\nnIter)\sqrt{m\log(T K/\delta)}},
\end{align}
then we have $\|\nnweight_{\epochNum}^{(s+1)}-\nnweight^{(0)}\|_2\leq\omega$. Note that the choice of $m,\omega$ satisfies the condition in Lemma~\ref{lemma:local_linear}. Thus we know $\phi_j(\xb;\nnweight)$ is almost linear in $\nnweight$, which leads to 
\begin{align}\label{eq:decop_phij_local_linear}
    |\phi_j(\xb;\nnweight^{(s+1)})|&\leq|\phi_j(\xb;\nnweight^{(s)})+\la\nabla\phi_j(\xb;\nnweight^{(s)}),\nnweight^{(s+1)}-\nnweight^{(s)}\ra|+C_{5}\omega^{4/3}\layerLen^{3}d^{-1/2}\sqrt{m\log m}\notag\\
    &\leq \sum_{h=0}^{s}\frac{C\sqrt{\log(TK/\delta)}}{h+1}+\eta\sqrt{dm}\|\nabla\cL(\nnweight^{(s)})\|_2 +2C_{5}\omega^{4/3}\layerLen^{3}d^{-1/2}\sqrt{m\log m}\notag\\
    &\leq \sum_{h=0}^{s}\frac{C_{0}\sqrt{\log(TK/\delta)}}{h+1}+C_{3}\eta\sqrt{dm}\sqrt{CT^2 d^4 m\log(TK/\delta)}\log\nnIter \notag\\
    &\qquad+2C_{5}\omega^{4/3}\layerLen^{3}d^{-1/2}\sqrt{m\log m}\notag\\
    &=\sum_{h=0}^{s}\frac{C_0\sqrt{\log(TK/\delta)}}{h+1}+\frac{\omega\sqrt{dm}}{\nnIter}+2C_{5}\omega^{4/3}\layerLen^{3}d^{-1/2}\sqrt{m\log m},
\end{align}
where in the second inequality we used the induction hypothesis \eqref{eq:phi_w_induction},   Cauchy-Schwarz inequality and Lemma \ref{lemma:bounded_gradient}, and the third inequality is due to  \eqref{eq:bounding_gradient_wt}. Note that the definition of $\omega$ in \eqref{eq:choose_omega} ensures that $\omega\sqrt{dm}<1/2$ and $\omega^{4/3}\layerLen^{3}d^{-1/2}\sqrt{m\log m}\leq m^{-1/6}T^{-6}\layerLen^{-5}d^{-1/2}\sqrt{\log m}\leq 1/\nnIter$ as long as $m\geq \nnIter^6$. Plugging these two upper bounds back into \eqref{eq:decop_phij_local_linear} finishes the proof of \eqref{eq:phi_w_induction}. 

Note that for any $t\in[T]$, we have $\nnweight_t=\nnweight_{\epochNum}^{(\nnIter)}$ for some $\epochNum=1,2,\ldots$. Since we have $\nnweight_t\in\BB(\nnweight,\omega)$, the gradient $\gb(\xb;\nnweight^{(0)})$ can be directly bounded by Lemma \ref{lemma:bounded_gradient}, which implies $\|\gb(\xb;\nnweight^{(0)})\|_F\leq C_{6}\sqrt{d\layerLen m}$. Applying \eqref{eq:phi_w_induction} with $s=\nnIter$, we have the following bound of the neural network function $\bphi(\xb;\nnweight_{\epochNum}^{(\nnIter)})=\bphi(\xb;\nnweight_t)$ for all $t$ in the $\epochNum$-th epoch
\begin{align*}
    \|\bphi(\xb;\nnweight_t)\|_2\leq C_{0}\sqrt{d\log(\nnIter)\log(TK/\delta)},
\end{align*}
which completes the proof. In this proof, $\{C_j>0\}_{j=0,\ldots,6}$ are constants independent of problem parameters.
\end{proof}

\subsection{Proof of Lemma \ref{lemma:confidence_bound}}
The following lemma characterizes the concentration property of self-normalized martingales. 
\begin{lemma}[Theorem 1 in \cite{abbasi2011improved}]\label{lemma:self_normalized_martingale}
Let $\{\xi\}_{t=1}^{\infty}$ be a real-valued stochastic process and $\{\xb_t\}_{t=1}^{\infty}$ be a stochastic process in $\RR^d$. Let $\cF_t=\sigma(\xb_1,\ldots,\xb_{t+1},\xi-1,\ldots,\xi_t)$ be a $\sigma$-algebra such that $\xb_t$ and $\xi_t$ are $\cF_{t-1}$-measurable.  Let $\Ab_t=\lambda\Ib+\sum_{s=1}^{t}\xb_s\xb_s^{\top}$ for some constant $\lambda>0$ and $S_t=\sum_{s=1}^{t}\xi_s\xb_i$. If we assume $\xi_t$ is $\nu$-subGaussian conditional on $\cF_{t-1}$, then for any $\eta\in(0,1)$, with probability at least $1-\delta$, we have
\begin{align*}
  \|S_t\|_{\Ab_{t}^{-1}}^2\leq 2\nu^2\log\bigg(\frac{\det(\Ab_t)^{1/2}\det(\lambda\Ib)^{-1/2}}{\delta}\bigg).
\end{align*}
\end{lemma}

\begin{proof}[Proof of Lemma \ref{lemma:confidence_bound}]
Let $\bPhi_{t}=[\bphi(\xb_{1,a_1};\nnweight_0),\ldots,\bphi(\xb_{t,a_t};\nnweight_{t-1})]\in\RR^{d\times t}$ be the collection of  feature vectors of the chosen arms up to time $t$ and $\hat\rb_t=(\empR_1,\ldots,\empR_t)^{\top}$ be the concatenation of all received rewards. According to Algorithm \ref{alg:deepUCB}, we have $\Ab_t=\lambda\Ib+\bPhi_t\bPhi_t^{\top}$ and thus
\begin{align*}
    \btheta_{t}=\Ab_{t}^{-1}\bbb_{t}=(\lambda\Ib+\bPhi_t\bPhi_t^{\top})^{-1}\bPhi_t\hat\rb_t.
\end{align*}
By Lemma \ref{lemma:linearization}, the underlying reward generating function $r_t=r(\xb_{t,a_t})=\EE[\empR(\xb_{t,a_t})|\xb_{t,a_t}]$ can be rewritten as
\begin{align*}
    r_t=\la\btheta^{*},\bphi(\xb_{t,a_t};\nnweight_{t-1})\ra+\btheta_0^{\top}\gb(\xb_{t,a_t};\nnweight^{(0)})(\nnweight^*-\nnweight^{(0)}).
\end{align*}
By the definition of the reward in \eqref{eq:nn_model} we have $\hat r_t=r_t+\xi_t$. Therefore, it holds that
\begin{align*}
    \btheta_{t}&=\Ab_t^{-1}\bPhi_t\bPhi_t^{\top}\btheta^*+\Ab_t^{-1}\sum_{s=1}^t \bphi(\xb_{s,a_s};\nnweight_{s-1})(\btheta_0^{\top}\gb(\xb_{s,a_s};\nnweight^{(0)})(\nnweight^*-\nnweight^{(0)})+\xi_s)\\
    &=\btheta^*-\lambda\Ab_t^{-1}\btheta^*+\Ab_t^{-1}\sum_{s=1}^t \bphi(\xb_{s,a_s};\nnweight_{s-1})(\btheta_0^{\top}\gb(\xb_{s,a_s};\nnweight^{(0)})(\nnweight^*-\nnweight^{(0)})+\xi_s).
\end{align*}
Note that $\Ab_t$ is positive definite as long as $\lambda>0$. Therefore $\|\cdot\|_{\Ab_t}$ and $\|\cdot\|_{\Ab_t}$ are well defined norms. Then for any $\delta\in(0,1)$ by triangle inequality we have 
\begin{align*}
    \| \btheta_{t}-\btheta^*-\Ab_t^{-1}\bPhi_t\bTheta_t\Gb_t(\nnweight^*-\nnweight^{(0)})\|_{\Ab_t}&\leq\lambda\|\btheta^*\|_{\Ab_t^{-1}}+\|\bPhi_t\bxi_t\|_{\Ab_t^{-1}}\\
    &\leq \nu\sqrt{2\log\bigg(\frac{\det(\Ab_t)^{1/2}\det(\lambda\Ib)^{-1/2}}{\delta}\bigg)}+\lambda^{1/2}M
\end{align*}
holds with probability at least $1-\delta$, where in the last inequality we used Lemma \ref{lemma:self_normalized_martingale} and the fact that $\|\btheta^*\|_{\Ab_t^{-1}}\leq\lambda^{-1/2}\|\btheta^*\|_2\leq\lambda^{-1/2}M$ by Lemma \ref{lemma:linearization}. Plugging the definition of $\bPhi_t, \bTheta_t$ and $\Gb_t$ and apply Lemma \ref{lemma:det_sum}, we further have
\begin{align*}
   &\bigg\| \btheta_{t}-\btheta^*-\Ab_t^{-1}\sum_{s=1}^t\bphi(\xb_{s,a_s};\nnweight_{s-1})\btheta_0^{\top}\gb(\xb_{s,a_s};\nnweight^{(0)})(\nnweight^*-\nnweight^{(0)})\bigg\|_{\Ab_t}\\
   &\leq \nu\sqrt{2\big(d\log(1+t(\log \bufferLen K)/\lambda )+\log 1/\delta\big)}+\lambda^{1/2}M,
\end{align*}
where we used the fact that $\|\phi(\xb;\nnweight)\|_2\leq C\sqrt{d\log \bufferLen K}$ by Lemma \ref{lemma:gradient_function_bound}.
\end{proof}

\subsection{Proof of Lemma \ref{lemma:A_inverse_sum_phi}}
We now prove the technical lemma that upper bounds $\|\Ab_t^{-1}\sum_{s=1}^t \bphi_s \zeta_s\|_2$.

\begin{proof}[Proof of Lemma \ref{lemma:A_inverse_sum_phi}]
We first construct auxiliary vectors $\tilde\bphi_t\in\RR^{d+1}$ and matrices $\Bb_t\in\RR^{(d+1)\times(d+1)}$ for all $t=1,\ldots$ in the following way:
\begin{align}\label{eq:def_lifted_phi}
    \tilde\bphi_t=\begin{bmatrix}
    G^{-1}\bphi_t\\
    \sqrt{1-G^{-2}\|\bphi_t\|_2^2}
    \end{bmatrix},\quad
    \Bb_t=\begin{bmatrix}
    \Ab_t^{-1}&\zero_d\\
    \zero_d^{\top} &0
    \end{bmatrix},
\end{align}
where $\zero_d\in\RR^d$ is an all-zero vector. Then by definition we immediately have
\begin{align}\label{eq:equiv_A_B}
     \bigg\|\Ab_t^{-1}\sum_{s=1}^t \bphi_s \zeta_s\bigg\|_2  =   \bigg\|\Bb_t\sum_{s=1}^t \tilde\bphi_s \zeta_s\bigg\|_2.
\end{align}
For all $s=1,2\ldots$, let $\{\beta_{s,j}\}_{j=1}^{d+1}$ be the coefficients of the decomposition of $U^{-1}\zeta_s\tilde\bphi_s$ on the natural basis. Specifically, let $\{\eb_1,\ldots,\eb_{d+1}\}$ be the natural basis of $\RR^{d+1}$ such that the entries of $\eb_j$ are all zero except the $j$-th entry which equals $1$. Then we have
\begin{align}
    U^{-1}\zeta_s\tilde\bphi_s = \sum_{j=1}^d \beta_{s,j}\eb_j, \quad\forall s=1,2,\ldots 
\end{align}
We can conclude that $|\beta_{s,j}| \leq 1 $ since $|\zeta_s|\leq U$ and $\|\tilde\bphi_s\|_2\leq 1$. Moreover,
it is  easy to verify that $\|\tilde\bphi_t\|_2 = 1$ for all $t \geq 1$. Therefore, we have
\begin{align}\label{eq:natural_basis_out}
    \bigg\|\Bb_t \sum_{s=1}^t \tilde\bphi_s\zeta_s\bigg\|_2 &= \bigg\|\Bb_t \sum_{s=1}^t \tilde\bphi_s \tilde\bphi_s^\top \tilde\bphi_s \zeta_s\bigg\|_2\notag \\
    &= \bigg\|\Bb_t \sum_{s=1}^t \tilde\bphi_s \tilde\bphi_s^\top U\sum_{j=1}^d \beta_{s,j} \eb_j\bigg\|_2 \notag \\
    &= U\bigg\|\sum_{j=1}^d\Bb_t \sum_{s=1}^t \tilde\bphi_s \tilde\bphi_s^\top  \beta_{s,j} \eb_j\bigg\|_2\notag \\
    & \leq U\sum_{j=1}^d\bigg\|\Bb_t \sum_{s=1}^t \tilde\bphi_s \tilde\bphi_s^\top  \beta_{s,j}\bigg\|_2\notag\\
    &=U\sum_{j=1}^d\bigg\|\Ab_t^{-1} \sum_{s=1}^t \bphi_s \bphi_s^\top  \beta_{s,j}\bigg\|_2,
\end{align}
where the inequality is due to triangle inequality and the last equation is due to the definition of $\tilde\bphi_t$ and $\Bb_t$ in \eqref{eq:def_lifted_phi}. For each $j=1,\ldots,d+1$, we have
\begin{align}\label{eq:decompose_beta_neg_posi}
    \bigg\|\Ab_t^{-1} \sum_{s=1}^t \bphi_s \bphi_s^\top  \beta_{s,j}\bigg\|_2&=\bigg\|\Ab_t^{-1} \sum_{s\in[t]:\beta_{s,j}\geq 0} \bphi_s \bphi_s^\top  \beta_{s,j}+\Ab_t^{-1} \sum_{s\in[t]:\beta_{s,j}<0} \bphi_s \bphi_s^\top  \beta_{s,j}\bigg\|_2\notag\\
    &\leq\bigg\|\Ab_t^{-1} \sum_{s\in[t]:\beta_{s,j}\geq 0} \bphi_s \bphi_s^\top  \beta_{s,j}\bigg\|_2+\bigg\|\Ab_t^{-1} \sum_{s\in[t]:\beta_{s,j}<0} \bphi_s \bphi_s^\top  (-\beta_{s,j})\bigg\|_2.
\end{align}
Since we have $|\beta_{s,j}|\leq1$, it immediately implies
\begin{align*}
    \Ab_t&=\lambda\Ib+\sum_{s=1}^t\bphi_s\bphi_s^{\top}\succ\sum_{s\in[t]:\beta_{s,j}\geq0}\bphi_s\bphi_s^{\top}\beta_{s,j},\\
    \Ab_t&=\lambda\Ib+\sum_{s=1}^t\bphi_s\bphi_s^{\top}\succ\sum_{s\in[t]:\beta_{s,j}<0}\bphi_s\bphi_s^{\top}(-\beta_{s,j}).
\end{align*}
Further by the fact that $\|\Ab^{-1}\Bb\|_2\leq 1$ for any $\Ab \succ \Bb \succeq 0$, combining the above results with \eqref{eq:decompose_beta_neg_posi} yields
\begin{align*}
    \bigg\|\Ab_t^{-1} \sum_{s=1}^t \bphi_s \bphi_s^\top  \beta_{s,j}\bigg\|_2\leq 2.
\end{align*}
Finally, substituting the above results into \eqref{eq:natural_basis_out} and \eqref{eq:equiv_A_B} we have
\begin{align*}
    \bigg\|\Ab_t^{-1}\sum_{s=1}^t \bphi_s \zeta_s\bigg\|_2 \leq 2Ud,
\end{align*}
which completes the proof. 
\end{proof}

\bibliographystyle{ims}
\bibliography{reference}
\end{document}